\documentclass{article}
\usepackage{bcc_arxiv,times}

\usepackage{amsmath,amsfonts,bm}

\def\eqref#1{equation~\ref{#1}}

\def\1{\bm{1}}

\DeclareMathAlphabet{\mathsfit}{\encodingdefault}{\sfdefault}{m}{sl}
\SetMathAlphabet{\mathsfit}{bold}{\encodingdefault}{\sfdefault}{bx}{n}

\newcommand{\E}{\mathbb{E}}

\newcommand{\KL}{D_{\mathrm{KL}}}

\usepackage{amsmath,amssymb,amsthm,mathtools,bm}
\usepackage{booktabs,tabularx,longtable,multirow,array}
\newcommand{\bccrowsep}{\specialrule{0.25pt}{1.1pt}{1.1pt}}
\usepackage{graphicx}
\usepackage{float}
\usepackage{placeins}
\usepackage{xcolor}
\usepackage{tikz}
\usetikzlibrary{arrows.meta,fit,positioning}
\usepackage{hyperref}
\hypersetup{colorlinks=true,linkcolor=blue!55!black,citecolor=blue!65!black,urlcolor=blue!65!black,pdftitle={Behavioral Capacity Certificates for Quantized Language Models},pdfauthor={Arian Eamaz and Mojtaba Soltanalian}}
\usepackage{url}
\usepackage[capitalize,noabbrev]{cleveref}

\usepackage{aliascnt}
\newtheorem{theorem}{Theorem}[section]
\newaliascnt{proposition}{theorem}
\newtheorem{proposition}[proposition]{Proposition}
\aliascntresetthe{proposition}
\newaliascnt{corollary}{theorem}
\newtheorem{corollary}[corollary]{Corollary}
\aliascntresetthe{corollary}
\newaliascnt{assumption}{theorem}

\aliascntresetthe{assumption}
\newcommand{\Risk}{\mathcal R}
\newcommand{\Rh}{\widehat{\mathcal R}}

\crefname{suppsection}{Supplementary Section}{Supplementary Sections}
\Crefname{suppsection}{Supplementary Section}{Supplementary Sections}
\crefname{supptable}{Supplementary Table}{Supplementary Tables}
\Crefname{supptable}{Supplementary Table}{Supplementary Tables}
\crefname{suppfigure}{Supplementary Figure}{Supplementary Figures}
\Crefname{suppfigure}{Supplementary Figure}{Supplementary Figures}
\crefname{suppproposition}{Supplementary Proposition}{Supplementary Propositions}
\crefname{supptheorem}{Supplementary Theorem}{Supplementary Theorems}
\crefname{suppcorollary}{Supplementary Corollary}{Supplementary Corollaries}
\crefname{suppequation}{Supplementary Equation}{Supplementary Equations}

\title{Behavioral Capacity Certificates\\for Quantized Language Models}

\author{Arian Eamaz \qquad Mojtaba Soltanalian\\[4pt]
Department of Electrical and Computer Engineering\\
University of Illinois Chicago\\[2pt]
\texttt{aeamaz2@uic.edu} \qquad \texttt{msol@uic.edu}}

\usepackage{microtype}
\usepackage{xurl}
\graphicspath{{figures/}}

\begin{document}

\maketitle
\begingroup
\renewcommand{\thefootnote}{}
\footnotetext{Code and reproduction instructions: \url{https://github.com/eamaz/bcc}.}
\endgroup

\begin{abstract}
Activation and key--value cache precision change what a quantized language model
computes without altering its stored weights. Direct weight-code bounds, however,
assign identical complexity to deployments that behave differently and charge
separately for weight codes that behave identically. Behavioral Capacity
Certificates (BCC) charge for behavior using the aggregate prior mass of complete
implementations---weights, scales, activation and cache rules---that induce the
same bounded loss. When quantization merges implementations, this shared mass
lowers the complexity penalty, and a break-even law determines when the saving
survives the cost of validating it. BCC supports a three-step deployment
workflow, and our experiments verify each step. First, a forward-only screen
shortlists per-layer bit-widths by how often candidate perturbations preserve the
reference predictions, with quality comparable to Hessian-guided selection at
lower preprocessing cost. Second, margin-certified cells identify weights that can be pruned or
sign-flipped without changing the deployed behavior: every permitted
combination preserves all declared predictions, and on OLMoE-1B-7B and
SmolLM2-1.7B, independent probes bound the probability that any permitted
combination changes a prediction on new text. Third, BCC bounds the population loss of the deployed model,
nonvacuously for complete decoders and more tightly than the
compressed-code route. At equal cache memory, giving keys higher
precision than values yields lower NLL and higher prediction
agreement on GPT-2, Qwen2.5, and SmolLM2, together with a tighter
complexity bound in the GPT-2 audit.
\end{abstract}

\section{Introduction}
Weight quantization reduces parameter storage of large language models~\citep{gptq2022,awq2024}.
But a deployment is not solely defined by a weight bit-width.
At long context, the key--value (KV) cache can dominate inference memory,
motivating its quantization alongside the weights~\citep{kivi2024}.
Activation quantization also helps meet computation budgets by enabling
low-precision matrix multiplications and reducing activation
storage~\citep{smoothquant}.
Activation precision, cache precision, scaling conventions, and clipping
rules shape the computation and can be varied while stored weight tensors
remain fixed. It is therefore plausible to expect scenarios in which deployments with byte-identical weight checkpoints compute different functions, while activation rounding absorbs weight differences, allowing distinct checkpoints to produce identical outputs.

Compression-based generalization bounds connect empirical loss to a
description of a reconstructed predictor, including PAC--Bayesian
compression approaches~\citep{zhou2019,lotfi2022,lotfi2024}.
When numerical conventions are fixed as side information, a direct
weight-code charge assigns identical complexity to a 16-bit and a 4-bit
cache even when their outputs differ. This motivates a complementary question to choosing the best
weight--activation (W/A) format for a transformer's feed-forward network
(FFN): \emph{how should generalization be certified when different W/A
formats and cache rules induce the same bounded loss on the declared
domain?} Behavioral Capacity Certificates address this question by
aggregating prior probability over complete implementations with the
same bounded loss function. Their shared mass determines the behavioral
complexity, allowing the bound to credit alternative implementations
of the same task behavior.
\paragraph{Background and prior art.}
PAC--Bayes relates population loss to empirical loss and divergence from a
sample-independent prior~\citep{shawewilliamson1997,mcallester1998,mcallester1999,maurer2004note}.
Compression analyses reduce that divergence by giving the predictor a short
description: noise-stable networks admit compressed surrogates~\citep{arora2018stronger},
and explicit codes yield nonvacuous bounds for stochastic networks and language
models~\citep{dziugaite2017,zhou2019,lotfi2024}. Our setting differs in two ways.
First, these bounds charge for the description of one predictor; we charge for
the pooled prior mass of every implementation inducing the same loss, so
activation and cache choices that leave the weights unchanged still enter the
certificate. Second, under a declared perturbation law that mass is the
probability that an implementation change leaves the loss unchanged, so one
quantity bounds both complexity and perturbation risk. Functional-equivalence
covers~\citep{shen2024functional} and symmetry-based PAC--Bayes
bounds~\citep{beck2025symmetry} exploit hidden-unit permutations or other group
actions that preserve the input--output map; we merge implementations across
quantization formats that share a loss on every declared input.
HAWQ-V2~\citep{hawqv2} and SmoothQuant~\citep{smoothquant} choose precisions with
curvature or rescaling heuristics defined on weights; we score candidates by the
behavior they preserve. 

\paragraph{Certifying shared behavior.}
Let $\mathcal E$ be the countable set of complete implementations, and let
$\Pi(e)=g_e$ map each implementation to its bounded loss function
$g_e:\mathcal X\to[0,1]$. The \emph{fiber} of a behavior $g$ is its preimage:

\[
\mathcal F(g):=\Pi^{-1}(g)
=\{e\in\mathcal E:\ g_e(X)=g(X)\ \text{for every }X\in\mathcal X\}.
\]
For a sample-independent prior $\nu$, its mass and \emph{behavioral complexity} are
\begin{equation}
\bar\nu(g)=\nu(\mathcal F(g))=\sum_{e\in\mathcal F(g)}\nu(e),\qquad
K_{W,A}(g)=-\ln\bar\nu(g).
\label{eq:behavioral_mass}
\end{equation}
The \emph{behavioral quotient} $\mathcal E/{\sim}$ is the set of
these fibers, grouping implementations that have identical loss
on every declared input; thus $e\sim e'$ exactly when $g_e=g_{e'}$.
The Occam--KL bound certifies deterministic deployments at this complexity. With a uniform $B$-bit prior, a size-$N_g$ fiber costs $B-\log_2 N_g$ bits. Certified subsets lower-bound its mass
and thus upper-bound its complexity. A three-implementation toy example can be a uniform prior over three complete implementations,
differing only in their quantized weight vectors:
$w_1=(-1,-1,+1)$, $w_2=(-1,+1,+1)$, and
$w_3=(+1,-1,+1)$, with $\nu(e_i)=1/3$.
Suppose $e_1$ and $e_2$ induce the same bounded loss function
$g_A$ on every declared input, while $e_3$ induces
$g_B\ne g_A$. The fiber $\mathcal F(g_A)=\{e_1,e_2\}$
therefore has mass $\bar\nu(g_A)=2/3$.
Naming $e_1$ costs $-\ln(1/3)=\ln3$, whereas naming its
behavior costs $K_{W,A}(g_A)=-\ln(2/3)=\ln(3/2)$.
BCC thus saves $\ln2$ nats, exactly one bit, in the
certificate's complexity penalty while keeping the same
empirical loss. This saving comes from shared behavior
and requires no parameter-space metric.
\begin{figure}[!t]
\centering
\begin{tikzpicture}[
 scale=1, transform shape,
 font=\scriptsize,>=Latex,
 box/.style={draw,rounded corners=2pt,align=center,inner xsep=3pt,inner ysep=4pt,minimum height=15mm,outer sep=0pt}
]
 \node[box,fill=gray!8,text width=28mm] (impl) {
  \textbf{Complete implementation $e$}\\[2pt]
  weights, scales, thresholds,\\residual and cache rules\\
  direct cost: $-\ln\nu(e)$};
 \node[box,fill=blue!6,text width=20mm,right=5mm of impl] (behavior) {
  \textbf{Behavior $g_e$}\\[2pt]
  bounded loss on the\\population domain};
 \node[box,fill=blue!10,text width=40mm,right=5mm of behavior] (complexity) {
  \textbf{Shared mass and complexity}\\[2pt]
  $\bar\nu(g)=\sum_{e':g_{e'}=g}\nu(e')$\\[2pt]
  $K_{W,A}(g)=-\ln\bar\nu(g)$};
 \node[box,fill=blue!4,text width=23mm,right=5mm of complexity] (bound) {
  \textbf{Risk certificate}\\[2pt]
  empirical loss\\plus complexity};
 \draw[->,thick] (impl)--(behavior);
 \draw[->,thick,blue!65!black] (behavior)--(complexity);
 \draw[->,thick,blue!65!black] (complexity)--(bound);
\end{tikzpicture}
\caption{BCC aggregates prior mass by bounded loss. Construction-derived
cells additionally pay independent-probe transfer.}
\label{fig:implementationmap}

\end{figure}
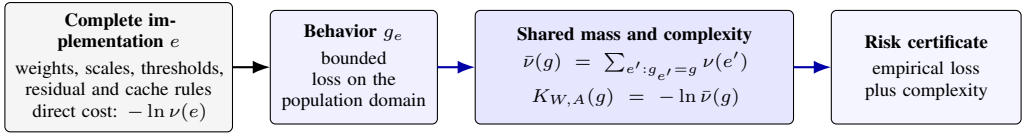
A practitioner chooses weight and activation
precision, compresses the KV cache, prunes weights, and must tolerate faults in
the stored ones; each of these changes the implementation, and a guarantee tied
to one implementation's description survives none of them. Because BCC
certifies behavior rather than a particular implementation, every change that
stays inside the certified fiber keeps the same guarantee, and the certificate identifies which changes stay inside it. This supports a simple workflow:
screen candidates cheaply, certify the changes that preserve behavior, and bound
the model finally deployed.
The main target is a certificate that stays valid through the decisions a
deployment actually makes. 
Our contributions are:
\begin{itemize}

\item \textbf{Nonvacuous bounds beyond lossless compression.}
BCC aggregates prior mass across complete implementations that induce
the same bounded loss function. We derive population-risk bounds and
identify when aggregation savings survive independent validation.
Complete-decoder experiments yield nonvacuous bounds tighter than
lossless histogram-and-order coding; exact finite-task comparisons
also improve on the cheapest Huffman-coded realization under the
same prior (\cref{tab:mainfullmodel,tab:compressioncomparison}).

\item \textbf{One mass for generalization and stability.}
Under a declared perturbation law, the probability of preserving the
entire loss function bounds both behavioral complexity and expected
perturbation damage, with the mixture-prior cost included.
The provided certificate does not need any parameter metric, differentiability, or local minimum
as discussed in \cref{prop:massstability}. \Cref{prop:coordinatecell} identifies weight signs that may be flipped and
weights that may be zeroed while every reference prediction on the declared
domain is preserved. Many such perturbations change the weights and not the
behavior: they remain in the same fiber, so a practitioner may deploy them and
keep the same certificate.
Margin certificates identify large families of simultaneous output-head
sign flips and pruning edits with prediction-preservation guarantees,
validated on GPT-2 and OLMoE across cache precisions
(\cref{tab:olmoe_head} and Appendix~\ref{core:cacheedits}).

\item \textbf{Practical guidance for quantized deployment.}
A forward-only screen shortlists weight-precision choices with
quality comparable to Hessian-guided selection at lower measured
preprocessing cost (\cref{tab:mainscreening}).
At equal cache memory, our audits show that assigning higher precision
to keys than values yields lower NLL and higher prediction agreement
on GPT-2, Qwen2.5, and SmolLM2
(\cref{tab:mainkvcache}; Appendix~\ref{app:modern_kv}).

\end{itemize}
\paragraph{Notation.}
An implementation $e=(a_q,\theta,s,t,r,k,m_T)$ specifies format, weights,
scales, thresholds, residual/normalization rules, cache rules, and
teacher/calibration metadata. Its predictor is $f_e$, and $e_0$ denotes the reference
implementation.
For cache comparisons, $e_{0,h}$ uses cache setting $h$. Let $S\sim\mathcal D^m$ and $Z\sim\mathcal D^n$ be independent document
samples for construction and probing. 
For $\ell(f_e,X)\in[a,a+\Delta]$, define
$g_e=(\ell(f_e,\cdot)-a)/\Delta$,
$R_0(g)=\E_{\mathcal D}g(X)$, and
$\widehat R_0(g)=m^{-1}\sum_i g(X_i)$.
Thus $\Risk(f_e)=a+\Delta R_0(g_e)$ and
$\Rh_S(f_e)=a+\Delta\widehat R_0(g_e)$.
For $q\in[0,1]$ and $c\ge0$,
$\operatorname{kl}^{-1}_+(q,c)
=\sup\{r\in[q,1]:\operatorname{kl}(q\|r)\le c\}$,
where $\operatorname{kl}$ is binary KL divergence.  In this paper, a \emph{certificate} is a mathematically established guarantee on
population loss or preservation of declared behavior, under the stated
assumptions.

\section{Joint W/A behavioral complexity}
\label{sec:framework}

We derive population guarantees for complete W/A behaviors and use them
to certify deployment choices. Starting from the behavioral mass in
\cref{eq:behavioral_mass}, \cref{thm:behavioral} establishes the population
bound, and \cref{cor:deployment} applies it to deployment targets.
We then compare the resulting complexity charge with direct
implementation coding.

Evaluating the full behavioral mass can require comparing many complete
implementations, so we develop certificates based on local audits and
verified subsets. \Cref{prop:routedfactorization} gives an exact
composition rule for routed architectures. For cells discovered from
data, \cref{thm:publicquotient} quantifies the independent-transfer cost,
and \cref{prop:quotientcrossover,cor:probebudget} determine when the
complexity savings outweigh this cost and how many probes are required
under a declared transfer scenario. A further construction uses margin
budgets to certify a product cell without enumerating its members
(\cref{prop:coordinatecell} in Section~\ref{sec:practitioner}). The next theorem turns behavioral mass into a population guarantee valid
after model selection.

\begin{theorem}[Behavioral Occam--KL certificate]
\label{thm:behavioral}
With probability at least $1-\delta$ over $S$, simultaneously for every
complete W/A behavior $g$,

\begin{equation}
R_0(g)\le \operatorname{kl}^{-1}_+\!\left(
\widehat R_0(g),
\frac{K_{W,A}(g)+\ln(1/\delta)}{m}\right).
\label{eq:behavioral_bound}
\end{equation}
A prefix code of length $B(e)$ with prior mass at least $2^{-B(e)}$
gives $K_{W,A}(g_e)\le B(e)\ln2$;
Appendix~\ref{app:proofs} derives a simpler explicit upper bound
using Pinsker's inequality.
\end{theorem}

The complexity penalty depends on the combined prior probability of
implementations with the same loss function. Because the guarantee holds
simultaneously for all such functions, it remains valid when the deployed
implementation is selected using the certification data.

\paragraph{Why certify damage?}
Deployment decisions also require controlling the loss added by
quantization relative to a reference model. Here $f_0=f_{e_0}$, where
$e_0$ may be a teacher implementation fixed independently of $S$.
Positive damage counts loss increases without allowing improvements on
other inputs to cancel them. The following corollary certifies this
bounded target and supports selection under a permitted degradation budget:

\begin{corollary}[Certificate-guided deployment]
\label{cor:deployment}
Let $f_0$ be a reference predictor fixed independently of $S$, and let
$\ell(\cdot,X)\in[a,a+\Delta]$. For each candidate implementation $e$,
define the positive-damage behavior $d_e(X)=
[\ell(f_e,X)-\ell(f_0,X)]_+/\Delta\in[0,1]$,
or any measurable $[0,1]$-valued upper bound on this quantity, and let
$K_d(d_e)$ be its push-forward complexity under a sample-independent
candidate prior.
With probability at least $1-\delta$, simultaneously for all candidates,

\begin{equation}
\Risk(f_e)-\Risk(f_0)\le \Delta U_d(e),\qquad
U_d(e)=\operatorname{kl}^{-1}_+\!\left(
\widehat R_0(d_e),
\frac{K_d(d_e)+\ln(1/\delta)}{m}\right).
\label{eq:damagecertificate}
\end{equation}
Consequently, selecting the cheapest deployment satisfying
$U_d(e)\le\epsilon$ is valid without a post-selection penalty beyond
the declared prior.
\end{corollary}

Positive clipped-NLL damage guides our GPT-2 allocation and scale-calibration
studies (Appendix~\ref{core:allocation}); Theorem~2.4 also accommodates
this bounded target when transferring a data-constructed posterior's
guarantee to a selected implementation.

\paragraph{Behavioral complexity and direct implementation codes.}
Use the sample-independent prior
$\nu(e)=\pi(a_q)P_{a_q}(e\setminus a_q)$ in
\cref{eq:behavioral_mass} to obtain the behavioral mass and complexity.
For a fixed format $c$, define the conditional complexity
$K(g\mid c)=
-\ln\sum_{e:a_q(e)=c,\,\Pi(e)=g}P_c(e\setminus c)$.
The joint quantity obeys the log-sum-exp identity
$K_{W,A}(g)=-\ln\sum_c\pi(c)e^{-K(g\mid c)}$. Conditional comparisons charge the format selector $-\ln\pi(c)$;
$K_{W,A}(g)$ combines all formats sharing the same loss behavior
under the declared prior. This push-forward exposes the difference from the implementation code.
For a reference $e_0$, define the behavioral compression gain

\begin{equation}
G_{\rm BCC}(e_0)=
\ln\frac{\sum_{e:\Pi(e)=\Pi(e_0)}\nu(e)}{\nu(e_0)}\ge0.
\label{eq:bqc_gain}
\end{equation}
Hence
$K_{W,A}(g_{e_0})=-\ln\nu(e_0)-G_{\rm BCC}(e_0)$.
For a uniform $B$-bit implementation family with behavioral multiplicity
$N_g$, this becomes $K_{W,A}=B\ln2-\ln N_g$.
Activation precision affects complexity through fiber prior mass
(multiplicity under a uniform prior) or coded activation metadata,
so BCC can distinguish A4 from A8 without counting transient activations
as persistent parameters.

\paragraph{From local audits to a full-model certificate.}
Direct evaluation of the mass in \cref{eq:behavioral_mass} can require
comparing many complete implementations. To reduce this work, we seek
a way to combine local fiber masses from separate component audits. Layer-wise auditing is not valid for dense models. 
In a sequential network, changing one layer can change the inputs to
later layers, so separate local loss equalities do not preserve the
final loss. However, mixture-of-experts (MoE) models route inputs to selected
experts~\citep{switch2022}; fixed top-one routing permits exact loss-fiber
composition when each input's loss depends only on its selected expert.
A learned router must be fixed with stable assignments, and top-$k$
mixing requires a joint transition certificate unless the stated
one-expert loss dependence holds.

\begin{proposition}[Exact composition under fixed routing]
\label{prop:routedfactorization}
Let $\mathcal X=\bigsqcup_{u=1}^U A_u$ be a fixed partition, with fixed
shared computation and router. For $e=(e_1,\ldots,e_U)$, assume
$g_e(X)=g_{u,e_u}(X)$ on $A_u$: the loss there depends only on expert $u$.
For a realized behavior $g$, put
$F_u(g)=\{e_u:g_{u,e_u}=g|_{A_u}\}$. Then

\begin{equation}
F(g)=\prod_u F_u(g),\qquad
\bar\nu(g)=\prod_u p_u(g),\qquad
K_{W,A}(g)=\sum_u-\ln p_u(g),
\label{eq:routedfactorization}
\end{equation}
where the mass identities use a sample-independent product prior
$\nu=\bigotimes_u\nu_u$ and $p_u(g)=\nu_u(F_u(g))>0$.
For local format mixtures,
$p_u(g)=\sum_f\pi_u(f)\nu_{u,f}(F_u(g))$.
The fiber identity holds for arbitrary joint priors; independence is
needed only for the probability product. Certified local subsets give
the corresponding mass lower bound. If the fixed shared fields have
mixture weight $\pi(a)$, their contribution gives
$K_{W,A}(g)\le-\ln\pi(a)-\sum_u\ln p_u(g)$.
\end{proposition}

Appendix~\ref{app:routedproof} gives the proof and architectural conditions.
The results of the routed-composition are reported in \cref{tab:mainroutedgrid} of Appendix~\ref{core:routed}.
\paragraph{Why independent transfer?}
A cell found on $S$ may disagree outside it.
\Cref{thm:publicquotient} certifies its posterior risk and transfers the
guarantee to the selected implementation using the probe sample
$Z=(Z_1,\ldots,Z_n)\sim\mathcal D^n$, independent of $S$.
Each \emph{probe} $Z_j$ is a fresh input used to bound population
loss discrepancy from the reference when exhaustive verification
is infeasible.
For a distribution $Q$ on implementations define
$\widehat R_0(Q)=\E_{e\sim Q}\widehat R_0(g_e)$ and
$\widehat d_Z(Q,e_0)=
n^{-1}\sum_j\E_{e\sim Q}|g_e(Z_j)-g_{e_0}(Z_j)|$.

\begin{theorem}[Posterior transfer and layerwise sample quotients]
\label{thm:publicquotient}
After observing $S$, let $e_0$ be arbitrary and let $Q_S$ be any posterior
on implementations. If $(e_0,Q_S)$ is fixed before observing $Z$, then
with probability at least $1-\delta-\eta$ over $(S,Z)$,

\begin{equation}
\label{eq:posteriortransfer}
R_0(g_{e_0})\le
\eta_Z+\operatorname{kl}^{-1}_+\!\left(
\widehat R_0(Q_S),
\frac{\KL(Q_S\|\nu)+\ln(2\sqrt m/\delta)}{m}\right),
\end{equation}

with
$\eta_Z=\operatorname{kl}^{-1}_+\!\left(
\widehat d_Z(Q_S,e_0),\frac{\ln(1/\eta)}{n}\right)$.
In particular, let $\mathcal C_S(e_0)$ be a measurable cell whose members
have the same loss vector as $e_0$ on $S$, set
$\mu_S=\nu(\mathcal C_S(e_0))>0$, and take
$Q_S=\nu(\cdot\mid\mathcal C_S(e_0))$.
Then $\widehat R_0(Q_S)=\widehat R_0(g_{e_0})$,
$\KL(Q_S\|\nu)=-\ln\mu_S$, and

\begin{equation}
R_0(g_{e_0})\le
\eta_Z+\operatorname{kl}^{-1}_+\!\left(
\widehat R_0(g_{e_0}),
\frac{-\ln\mu_S+\ln(2\sqrt m/\delta)}{m}\right).
\label{eq:publicquotient}
\end{equation}
If the prior factorizes over layers and fixed metadata and
$\mathcal C_S=\prod_{\ell=1}^L\mathcal C_{\ell,S}$, then
$-\ln\mu_S=\sum_\ell-\ln\nu_\ell(\mathcal C_{\ell,S})$
plus the metadata cost.
\end{theorem}

\Cref{thm:publicquotient} converts a construction-sample cell into a
population guarantee for $e_0$, with $\eta_Z$ bounding the additional
risk incurred when replacing the posterior by that implementation.
Thus, observed merging can support certification even when cell members
differ on new inputs. Exact sample cells use the reference's empirical
loss; approximate posteriors use their own empirical mean loss.
Selection using the probes requires another split or a selector penalty.
Layerwise composition still requires complete transition equality in
dense networks, or the routed conditions of
\cref{prop:routedfactorization}.

\paragraph{Why a break-even law?}
A quotient lowers the complexity term, but it does not automatically lower
the bound. The quotient posterior is charged its own empirical loss, which
may exceed the reference implementation's; it carries a PAC--Bayes
confidence term in place of the direct code's; and a cell built from
construction data also pays a certified transfer $\tau$.
The next proposition combines these three charges into an exact threshold
$G_\star$: the transferred quotient improves the direct bound precisely
when the saving exceeds it. Corollary~\ref{cor:probebudget} converts that
threshold into a probe count, so the required budget can be assessed
under a declared transfer scenario before probes are drawn.

\begin{proposition}[Exact quotient pay-for-it criterion]
\label{prop:quotientcrossover}
Let a direct comparator use $B$ bits, empirical loss $q_0$, and failure
probability $\alpha$, and let a quotient posterior use empirical Gibbs
loss $q_Q$, saving $G\in[0,B]$ bits, PAC--Bayes failure probability $\delta$,
and an independently certified additive transfer $\tau$. Let $U_{\rm raw}=\operatorname{kl}^{-1}_+\!\left(
q_0,\frac{B\ln2+\ln(1/\alpha)}{m}\right)$. If $U_{\rm raw}-\tau\le q_Q$, the transferred quotient cannot be strictly
smaller. If $q_Q<U_{\rm raw}-\tau<1$, then

\begin{equation}
U_{\rm BCC}<U_{\rm raw}
\quad\Longleftrightarrow\quad
G>G_\star:=
B+\log_2\!\frac{2\sqrt m}{\delta}
-\frac{m}{\ln2}
\operatorname{kl}\!\left(q_Q\middle\|U_{\rm raw}-\tau\right).
\label{eq:quotientcrossover}
\end{equation}
Appendix~\ref{app:proofs} separates this threshold into confidence and
fit--transfer costs. Returning the smaller of both bounds requires
joint failure budget $\alpha+\delta+\eta$ when $\tau$ has failure
probability $\eta$.
\end{proposition}

\begin{corollary}[Independent-probe budget for a quotient win]
\label{cor:probebudget}
Fix the direct bound $U_{\rm raw}$ and the quotient core

\[
U_{\rm core}=\operatorname{kl}^{-1}_+\!\left(
q_Q,\frac{(B-G)\ln2+\ln(2\sqrt m/\delta)}{m}\right),
\qquad
\Delta_{\rm core}=U_{\rm raw}-U_{\rm core}.
\]
For a probe count fixed before observing $Z$ and its valid transfer bound
$\tau_n$, the transferred quotient is strictly better when
$\Delta_{\rm core}>0$ and $\tau_n<\Delta_{\rm core}$. Consequently $n_\star=
\min\{n\in\mathbb N_{\ge1}:\tau_n<\Delta_{\rm core}\}$
is the break-even budget for a declared transfer scenario. In the zero-disagreement scenario, for
$0<\Delta_{\rm core}<1$ and transfer failure probability $\eta$,
$\tau_n=1-\eta^{1/n}$ and 

\begin{equation}
n_\star=\left\lfloor
\frac{\ln(1/\eta)}{-\ln(1-\Delta_{\rm core})}
\right\rfloor+1.
\label{eq:zerodisagreementbudget}
\end{equation}
Larger certified fibers decrease $U_{\rm core}$ and the required
zero-disagreement budget; Appendix~\ref{app:proofs} gives the derivative and
limiting-transfer analysis.
\end{corollary}
%

\section{Evidence for behavioral complexity and stability}
\label{sec:experiments}

We test exact merging, certificate tightening after validation, and stability
under specified implementation changes.

\paragraph{Protocol and scope.}
Local audits use a predeclared finite universe of complete implementations
that vary specified layers or operations while fixing the remaining
fields. This makes prior masses computable under a common comparison
prior; local-prior guarantees are conditional on those fixed fields.
An exact sample cell $C_S(e_0)$ contains implementations satisfying
$g_e(X_i)=g_{e_0}(X_i)$ for every $X_i\in S$.
A uniformly verified approximate cell lies in
$E_{\eta_0}(e_0)=
\{e:\sup_{X\in\mathcal X}|g_e(X)-g_{e_0}(X)|\le\eta_0\}$.
Cells inferred from construction data require independent transfer
(\cref{thm:publicquotient}); cells verified over the entire declared
population require no transfer probes.

\subsection{Exact merging and net certificate gains}
\paragraph{Lower weight and activation precision can preserve behavior.}
A direct code charges each implementation separately; the question is
whether implementations at different budgets that share the same behavior
can be charged as one.
We enumerate all $2^9=512$ binary and $3^9=19{,}683$ ternary
$3\times3$ weight matrices and evaluate their predictions on all
125 inputs.
For both binary and ternary weights, distinct weight assignments
evaluated at A4 and A8 can preserve every prediction on this complete
population (\cref{tab:exactquotient}).
Each row fixes a task and weight format and pools A4/A8 implementations
with equal prior weight.
Individual costs of 10.000 and 15.265 bits become shared behavioral
costs of 8.415 and 11.943 bits for the binary and ternary tasks,
respectively. Appendix~\ref{core:exact} also reports matching behavior across
weight formats: one W1/A8 and two W1.58/A8 implementations have
identical 0--1 loss vectors.
The binary implementation also matches at A4, preserving predictions
across both precision reductions. On three fixed decoder checkpoints, 18--29 W/A/K/V format choices
preserve all four-token continuations on 256 construction prefixes,
reducing the conditional complexity charge by 4.17--4.86 bits
(see Appendix~\ref{app:format_only}).
BCC further tightens compressed-code bounds under both a
histogram-Huffman code and a distribution-sensitive ternary prior,
relative to the cheapest individual realization of the shared
behavior under each prior
(\cref{tab:compressioncomparison};
Appendix~\ref{app:ternary_mixture}).
\begin{table}[!t]
\caption{Shared behavior across activation precisions.
Each row fixes a separate task on 125 inputs and a weight format.
A4/A8 entries give matching weight assignments divided by the total.
The prior is uniform over assignments and gives A4/A8 equal weight.
Direct cost names one implementation; BCC cost names the shared
behavior using the pooled prior mass. Costs are in bits.}
\label{tab:exactquotient}
\label{tab:scaleone}
\centering\footnotesize
\setlength{\tabcolsep}{4pt}
\begin{tabular}{@{}lrrrrr@{}}
\toprule
Weights
& \shortstack{A4 fiber/total}
& \shortstack{A8 fiber/total}
& \shortstack{Pooled prior mass}
& \shortstack{Direct cost}
& \shortstack{BCC cost}\\
\midrule
Binary
& $2/512$ & $1/512$ & $3/1024$
& 10.000 & \textbf{8.415}\\
\bccrowsep
Ternary
& $6/19{,}683$ & $4/19{,}683$ & $5/19{,}683$
& 15.265 & \textbf{11.943}\\
\bottomrule
\end{tabular}

\end{table}

\begin{table}[!t]
\caption{Nonvacuous complete-decoder NLL bounds (bits/target;
$m=530{,}000$; joint 95\% coverage). Prefixes are drawn independently with replacement
from the declared finite population, so \cref{thm:behavioral}'s independence
hypothesis holds with respect to that population.}
\label{tab:mainfullmodel}
\centering\footnotesize
\setlength{\tabcolsep}{3pt}       
\renewcommand{\arraystretch}{1} 

\begin{tabular}{@{}lrrrrr@{}}
\toprule
Format & Seed & $G$ (bits) & Literal & Compressed & BCC\\
\midrule
W4/A4 & 191 & 3206.4 & 0.6838 & 0.6328 & 0.6017\\
\bccrowsep
W4/A8 & 191 & 1459.3 & 0.7044 & 0.6647 & 0.6505\\
\bccrowsep
W4/A4 & 193 & 3107.1 & 0.6802 & 0.6285 & 0.5985\\
\bccrowsep
W4/A8 & 193 & 1550.2 & 0.7177 & 0.6723 & 0.6572\\
\bccrowsep
W4/A4 & 197 & 3125.0 & 0.6839 & 0.6355 & 0.6052\\
\bccrowsep
W4/A8 & 197 & 1440.6 & 0.6810 & 0.6363 & 0.6224\\
\bottomrule
\end{tabular}

\end{table}

\paragraph{Useful bounds for complete models.}
Complete codes give nonvacuous error and smoothed-NLL bounds in six
controlled four-layer runs (\cref{tab:mainfullmodel}) and six natural-text
deployments (\cref{tab:naturalcomplete} in Appendix~\ref{core:naturalcomplete}). BCC improves the compressed bound
in every row. Exact input-transition cells preserve all logits, without
transfer probes (Appendix~\ref{core:complete}).
Construction-selected 0--1 cells and Brier posteriors improve decoder
bounds in three replications each; probe-budget controls appear in \cref{tab:maincertificates} in Appendix~\ref{core:transfer}.
All four GPT-2 FFN local-prior bounds tighten after transfer
(Appendix~\ref{core:gpt}).

\begin{table}[!t]
\caption{Cache memory, task quality, and certified complexity at a
common GPT-2 checkpoint: 64 WikiText windows, 1,008-token prefill,
and 16 steps. Each row certifies a subset of its reference
deployment's behavioral fiber by varying bias codes at a fixed
cache format. This supplies $G$ bits of complexity reduction,
giving $K_{W,A}/\ln2\le B-G$ under the common code cost $B$.}
\label{tab:mainkvcache}
\centering\footnotesize
\setlength{\tabcolsep}{3pt}       
\renewcommand{\arraystretch}{1}  

\begin{tabular}{@{}lrrrr@{}}
\toprule
Cache & MiB/stream & NLL (nats) & Agree. (\%) & $G$ (bits)\\
\midrule
FP16/FP16 & 36.000 & 3.2137 & 100.00 & 1079.91\\
\bccrowsep
K8/V8 & 19.125 & 3.2135 & 99.22 & 398.19\\
\bccrowsep
K4/V8 & 14.625 & 3.2966 & 85.06 & 405.07\\
\bccrowsep
\textbf{K8/V4} & \textbf{14.625} & \textbf{3.2139} & \textbf{95.80} & \textbf{593.00}\\
\bccrowsep
K4/V4 & 10.125 & 3.2868 & 84.38 & 599.88\\
\bottomrule
\end{tabular}

\end{table}

\paragraph{Matched-memory cache finding: favor key precision.}
At 14.625 MiB and unchanged checkpoint bytes, K8/V4 beats K4/V8 on NLL,
agreement, and certified credit (+187.93 bits; \cref{tab:mainkvcache}).
This lower complexity upper bound holds at both prefix lengths
(Appendix~\ref{core:cache}). \emph{This preference agrees with prior reports that keys are more sensitive to
quantization}~\citep{tang2025spindlekv,kivi2024,hariri2026quantize}.

\paragraph{Memory, quality, and certified complexity.}
Results reported in \cref{tab:mainkvcache} show that K8/V8 saves 46.9\% of FP16 cache memory with near-identical NLL,
but provides 681.72 fewer bits of certified complexity reduction. In this experiment $G=\log_2 \left|\mathcal{C}\right|$, where $\mathcal{C}$ is the certified family of bias-code combinations that all produce identical logits at that cache setting, giving
the complexity upper bound $B-G$.
The key/value decomposition explains the observed ordering
(\cref{tab:kvcachecells} in Appendix~\ref{core:cache}): changing V8 to V4 adds 194.81 credit
bits, whereas changing K8 to K4 adds only 6.88 bits.
The coarser formats admit more bias-code combinations on the
same variable coordinates.
FP16 uses a different construction: it preserves half words
directly, while the low-bit construction also preserves shared
dynamic scales by freezing observed maximizers.
The resulting FP16 cell permits 251 key-bias coordinates to vary,
compared with only one under either low-bit key format.
These credits describe certified subsets of each setting's own
behavioral fiber; their ordering does not necessarily follow prediction
agreement or NLL.
The equal-memory quality advantage of higher key precision also
holds in our Qwen2.5 and SmolLM2 audits
(Appendix~\ref{app:modern_kv}).

\section{Certified families and deployment decisions}
\label{sec:practitioner}

We consider \emph{weight edits} such as sign flips ($w\mapsto -w$),
pruning ($w\mapsto 0$), and magnitude reductions that replace a
quantized weight with an allowed value of smaller magnitude.
A certified edit \emph{family} specifies which weights may change and
their allowed values.
The preceding bounds certify the nominal model; practitioners also need to know how implementation edits affect that guarantee. This requires explicit families of tolerated changes and risk bounds
under the intended perturbation law. We derive these using prediction margins and retained cell
mass, then connect them to screening, pruning, and fault audits. 

\begin{samepage}
\paragraph{Certifying simultaneous weight edits.}
Testing individual edits cannot establish that their combinations are
safe. A Hamming ball $\mathcal{B}_r(e_0)$ also includes every edit within its radius $r$, so
one sensitive coordinate can invalidate the whole family
(\cref{fig:behavioralfibers}A).
We instead select editable coordinates and bound how their joint changes
can close each winner--competitor logit gap.
\Cref{prop:coordinatecell} turns these margin budgets into a product cell
whose every permitted combination preserves predictions, without
enumerating its members. This supplies the guarantees used for sign-flip
assurance, pruning, and audit reduction.
We also test sign flips and pruning on GPT-2; Appendix~\ref{core:head}
reports these experiments.

\begin{proposition}[Margin-certified coordinate cells]
\label{prop:coordinatecell}
Fix the representation $\phi$, scale $\alpha>0$, and tie rule of a head
$\alpha c$ with $N$ symbols in an alphabet $\mathcal A$ of size $q$.
Let $v_j=\sup_{X\in\mathcal X}|\phi_j(X)|<\infty$. For editable coordinates
$J_k$ in row $k$, define
$L_k=\sum_{j\in J_k}\max_{a\in\mathcal A}|a-c_{kj}|v_j$.
For every input, let $w$ be its winner and
$m_{wk}=\langle c_w-c_k,\phi(X)\rangle$. If, for every $k\ne w$, $L_w+L_k<m_{wk}\quad\text{or}\quad L_w=L_k=0$, all $q^{D_*}$ joint edits, $D_*=\sum_k|J_k|$, preserve every prediction.
For prediction-error loss, an independent uniform head prior gives
$K_{W,A}(g)\le K_{\rm upstream}+(N-D_*)\ln q$, including upstream
and metadata cost. Proof: Appendix~\ref{app:coordinateproof}.
\end{proposition}
\end{samepage}

\begin{figure}[!tbp]
\centering
\includegraphics[width=\linewidth]{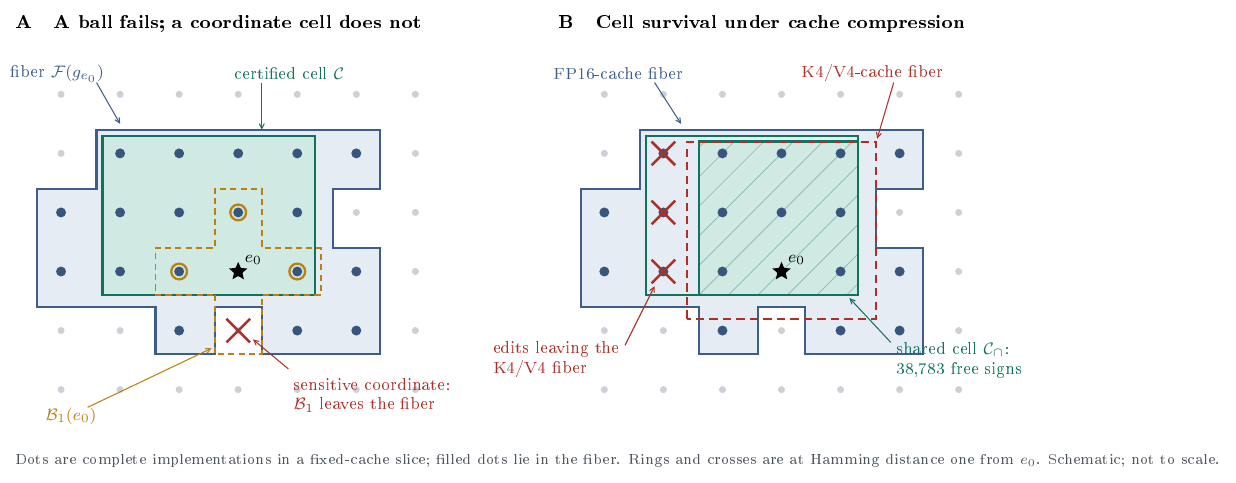}
\caption{Behavioral fibers under weight perturbations.
\textbf{A:} A weight perturbation can move the model outside its
reference fiber and change its predictions. The certified cell
contains combinations of perturbations that stay inside.
\textbf{B:} Cache precision changes the fibers. The shared cell
preserves each cache setting's own reference predictions.}
\label{fig:behavioralfibers}

\end{figure}

\paragraph{Population fibers without probe transfer.}
We test whether large families of simultaneous output-head edits can
preserve every declared prediction and reduce the complexity charged
by BCC. On the complete declared population of 23,735 GPT-2 contexts,
margin-certified head cells preserve every prediction and yield
explicit behavioral complexity savings under a uniform head prior:
50,623 bits for binary weights and approximately 75,565 bits for
ternary weights (for details see Appendix~\ref{core:head}).

\paragraph{Weight edits across cache precisions.}
Cache quantization changes the head's inputs and can reduce the margins
protecting weight edits. We therefore check whether edits certified
before cache compression still preserve the unedited model's predictions
at each cache setting. In the GPT-2 binary/A8 head experiment on 23,735
declared contexts, the full 50,623-sign bank remains safe at K8/V8,
whereas stronger compression invalidates some permitted edits.
A shared subset of 38,783 signs (76.61\%) remains certified across
all five tested cache formats; Appendix~\ref{core:cacheedits}
reports the matched-memory K8/V4 versus K4/V8 comparison.
\Cref{fig:behavioralfibers}B illustrates this restriction: the shared
cell preserves each cache setting's own unedited predictions, so its
weight edits add no prediction changes to those caused by cache compression.

The following modern-model experiments test whether BCC can certify large
families of weight changes at realistic model scale. The family's
prior mass determines its complexity charge, while independent
probes bound the probability that any permitted change alters a
reference prediction on a new input. For prediction-error loss,
this disagreement bound controls the transfer cost in
\cref{thm:publicquotient}.

\paragraph{Certified output-head edits.}
Using the margin construction of \cref{prop:coordinatecell}, we
certify simultaneous edit families so that the prediction-risk
guarantee applies to any permitted sign-flip pattern or pruning
subset. On \textbf{OLMoE-1B-7B}, one shared bank of 27,550 output-head
coordinates covers all $2^{27{,}550}$ sign patterns across the
K8/V8 and K8/V4 deployments. Margin bounds certify their joint
effect without enumerating the patterns. Evaluation on 8,192 independent prefixes
bounds the probability that any allowed edit changes a required
prediction by $0.217\%$ and $0.260\%$, respectively, relative to
each deployment's own unedited head. Both deployments meet all
seven certification criteria and have lower failure risk than
equal-size magnitude and activation-aware banks
(\cref{tab:olmoe_head}). Every head in the selected family also
has a population added clipped-NLL bound below $0.0233$ nats
per token. The cache savings come from KV quantization;
the edit certificate applies to the output head with expert and
router weights fixed. Appendix~\ref{app:olmoe_head} gives the
construction, confidence ledger, and complete diagnostics. At context length 1,024, the margin construction certifies a
shared family of 21,874 output-head coordinates on \textbf{SmolLM2-1.7B},
covering all $2^{21{,}874}$ sign patterns and their permitted
pruning combinations. On 8,192 independent prefixes, the
prediction-risk upper bounds are $0.121\%$ for K8/V8 and
$0.292\%$ for K8/V4. Both deployments meet all seven certification
criteria and have lower all-pattern failure risk than equal-size
magnitude and activation-cost controls. This extends the
edit-certificate evidence to a modern dense model at longer
context (Appendix~\ref{app:smollm2_head}).

\paragraph{Certified edit capacity on SmolLM2.}
Beyond comparing safety at a fixed family size, we ask how many
output-head weights can change together under a fixed risk limit.
At a 1\% continuation-disagreement limit, BCC certifies 32,768
coordinates jointly across K8/V8 and K8/V4, versus 1,024 for
activation-cost selection: a $32\times$ gain on the declared grid, while
magnitude selection passes no tested size
(Appendix~\ref{app:smollm_capacity}). For losses determined by the generated tokens and target,
the continuation-disagreement bound controls population transfer
(\cref{thm:publicquotient}).
\begin{table}[!tbp]
\centering\footnotesize
\setlength{\tabcolsep}{4pt}
\caption{OLMoE output-head edit certification at context length 256,
with four teacher-forced and four greedy continuation steps.
All banks contain 27,550 editable coordinates.
Advantages are simultaneous lower bounds on control failure risk
minus selected-family failure risk, in percentage points (pp).
All bounds share $98.33\%$ confidence. Cache savings include scales.}
\label{tab:olmoe_head}

\begin{tabular}{@{}lrrrrr@{}}
\toprule
Cache & \shortstack{Failed\\prefixes} & \shortstack{Risk upper\\(\%)}
& \shortstack{Magnitude\\advantage (pp)}
& \shortstack{Activation-cost\\advantage (pp)}
& \shortstack{Cache saved\\(\%)}\\
\midrule
K8/V8 & 4/8192 & 0.217 & 0.978 & 0.689 & 48.44 \\
K8/V4 & 6/8192 & 0.260 & 0.823 & 0.625 & 60.94 \\
\bottomrule
\end{tabular}

\end{table}

\paragraph{From certified edits to perturbation-risk bounds.}
\Cref{prop:coordinatecell} identifies a safe family; its probability under
a declared perturbation law measures how often an edit is guaranteed
to preserve behavior.
Unlike geometric flatness, which can change under function-preserving
reparameterizations~\citep{dinh2017sharp}, this probability requires
no parameter metric, differentiability, or local minimum.
When a sample-independent coding prior mixes the declared perturbation
laws, \cref{prop:massstability} uses fiber probability to bound both
behavioral complexity and expected loss increase, accounting for the
mixture weight. Certified cells supply lower bounds on
that probability.

\begin{proposition}[Fiber mass, a charged kernel prior, and stability]
\label{prop:massstability}
Fix distributions $\kappa_j$ and positive weights $\pi_j$ summing to one,
independently of the certification sample, and set
$\nu=\sum_j\pi_j\kappa_j$. For the population fiber $F_e$, let
$s_j=\kappa_j(F_e)>0$, $L_j=-\ln\pi_j$, and $R(e)=\mathbb E_X g_e(X)$.
With independent draws of $e'\sim\kappa_j$ and $X$,

\begin{equation}
\label{eq:kernelcode}
 K_{W,A}(g_e)\le L_j-\ln s_j,,\quad
 D_j^+(e):=\mathbb E\bigl[(g_{e'}(X)-g_e(X))_+\bigr]
 \le(1-s_j)(1-R(e)).
\end{equation}
Writing $R_{\kappa_j}(e)=\mathbb E_{e'\sim\kappa_j}R(e')$, if
$R(e)\le U\le1$, then $R_{\kappa_j}(e)\le1-s_j+s_jU$.
For $\nu=\kappa_j$, $K_{W,A}(g_e)=-\ln s_j$. Certified positive lower
bounds on $s_j$ retain these inequalities (see Appendix~\ref{app:massstability}).
\end{proposition}

\paragraph{Training and behavioral mass.}
Training affects the approximate cell mass under fixed quantization and
perturbation rules. In
Appendix~\ref{app:qat_mass}, Tables~\ref{tab:qat_mass_results}--\ref{tab:qat_mass_tolerance}, we compare quantization-aware
training (QAT) from scratch with deterministic post-training
quantization (PTQ), matched by format and initialization seed, across
four low-bit formats and three seeds. At mean absolute Brier tolerance
$1/1024$, construction-cell coverage averages 86.0\% under QAT and
3.2\% under PTQ over the declared nontrivial single-symbol edits.
QAT also yields lower independently evaluated positive Brier damage
in all 12 format--seed comparisons.
For fixed nominal-risk and within-cell discrepancy bounds $U$ and
$\tau$, larger mass tightens the perturbed-risk bound in
\cref{prop:coveredrisk} whenever $U+\tau<1$.
These measurements associate training with greater approximate cell
mass and perturbation robustness.

\begin{table}[!tbp]
\caption{Screening twelve four-layer references at three weight budgets
with FP32 activations. Alice reuses Shakespeare-selected schedules.
BFMS takes 3.40--6.15\,s versus 17.58--34.13\,s for Hessian. Disagreement is to FP; $\Delta$NLL is in nats.}
\label{tab:mainscreening}
\centering\footnotesize
\setlength{\tabcolsep}{7pt}
\begin{tabular}{@{}lrrrr@{}}
\toprule
& \multicolumn{2}{c}{Disagreement (\%)} & \multicolumn{2}{c}{$\Delta$NLL}\\
\bccrowsep
Selector & Shakespeare & Alice & Shakespeare & Alice\\
\midrule
BFMS: natural contexts & 39.61 & 39.90 & 0.3936 & 0.3415\\
\bccrowsep
HAWQ-style trace & 40.40 & 40.38 & 0.3974 & 0.3441\\
\bccrowsep
Covariance-aware Hessian & 40.12 & 40.38 & 0.3907 & 0.3402\\
\bccrowsep
Weight-MSE & 45.56 & 45.72 & 0.6340 & 0.5307\\
\bottomrule
\end{tabular}

\end{table}

\paragraph{Screening implementation choices.}
Because the space of layer precisions, scales and clipping rules is too large to evaluate exhaustively, behavioral fiber mass sensitivity (BFMS) screens layer/candidate pairs by forward-only FP prediction retention under declared perturbations.
It gives allocation quality comparable to Hessian-guided selection at lower
measured cost (\cref{tab:mainscreening}). Appendix~\ref{core:screen} specifies the
unlabeled probes, timing controls, joint W/A extension, and subsequent certification. Sequential damage allocation improves bounds, held-out damage, and NLL over
HAWQ at matched cost in both families (\cref{tab:practitionerdeployment} in Appendix~\ref{core:allocation}).
Binary joint scale--allocation selection reduces held-out damage by 19.2\%
and NLL increases by 37.4\% against checkpoint-MSE selection
(\cref{tab:scalecalibrationfull} in Appendix~\ref{core:allocation}). Certified zero-damage mass also reduces the sampling cost of
perturbation audits (Appendix~\ref{core:audit}).


\section{Conclusion}

BCC charges a deployment for its behavior rather than for one
implementation's description, so a single certificate covers every
implementation in the certified fiber. Verified transitions, fixed
routing, margin budgets, and independent probes make this mass
computable. The experiments use it to screen precisions, certify families
of weight edits on GPT-2, OLMoE, and SmolLM2, and compare cache
allocations at equal memory.

\bibliography{references}
\bibliographystyle{bcc_references}

\appendix
\raggedbottom
\clearpage
\section{Proofs and certified constructions}
\label{app:proofs}
This appendix proves every theorem, proposition, and corollary stated in
Sections~2 and~4, together with the re-centering inequality and the audit
budget used in the numerical results. All logarithms are natural unless a
base is shown; the main text defines the bounded losses and sampling events.
\begin{table}[H]
\caption{Location of the main results' proofs. The cited pages are in this
PDF.}
\label{tab:proofmap}
\centering\small
\begin{tabularx}{\textwidth}{@{}lXr@{}}
\toprule
Main result & Proof or construction & Page\\
\midrule
Proposition~\ref{prop:routedfactorization} & Fixed-routing factorization and feasible-set inclusion & \pageref{app:routedproof}\\
\bccrowsep
Theorem~\ref{thm:behavioral}; Corollary~\ref{cor:deployment} & Countable behavioral prior and positive damage & \pageref{proof:occam}\\
\bccrowsep
Theorem~\ref{thm:publicquotient} & Shared PAC--Bayes event and independent transfer & \pageref{proof:transfer}\\
\bccrowsep
Proposition~\ref{prop:quotientcrossover} & Exact inversion of the comparison & \pageref{proof:crossover}\\
\bccrowsep
Corollary~\ref{cor:probebudget} & Strict probe-budget threshold and integer rounding & \pageref{proof:probebudget}\\
\bccrowsep
Proposition~\ref{prop:coordinatecell} & Simultaneous coordinate edits and half-margin budgets & \pageref{app:coordinateproof}\\
\bccrowsep
Proposition~\ref{prop:massstability} & Coding mass, perturbation risk, and recoding invariance & \pageref{app:massstability}\\
\bccrowsep
Equation~\ref{eq:coveredrisk} & Covered-risk decomposition & \pageref{prop:coveredrisk}\\
\bccrowsep
Audit variance and budget & Proposition~\ref{prop:residualaudit} and Corollary~\ref{cor:auditplanning} & \pageref{app:residualtheory}\\
\bottomrule
\end{tabularx}
\end{table}

\subsection{Coding and entropy interpretation}
\label{app:codingbackground}
\begin{proposition}[Coding interpretation and entropy baseline]
\label{prop:canonicalrepresentative}
For every countable implementation prior $\nu$ and behavior map $\Pi$, fix one representative $s(g)\in\Pi^{-1}(g)$ for each $g$ with $\bar\nu(g)>0$ and define the sample-independent prior $\nu^\dagger(s(g))=\bar\nu(g)$, with zero mass on the other implementations.  Then the ordinary implementation-level Occam complexity of $s(g)$ is exactly
$-\ln\nu^\dagger(s(g))=K_{W,A}(g)$, and $s(g)$ has the same bounded loss function as every implementation in its fiber.  Thus BCC is the direct Occam code of an \emph{oracle} codebook containing one canonical representative per behavior.  It may strictly improve coding under the original declared implementation prior, but an oracle representative prior can match, rather than be strictly beaten by, the quotient.

More generally, let $E$ be the random implementation returned by any predeclared selection procedure, let $G=\Pi(E)$, and assume $H(E)<\infty$.  Among all sample-independent direct and behavioral priors,
\begin{equation}
\inf_\rho\E[-\ln\rho(E)]-\inf_\sigma\E[-\ln\sigma(G)]
=H(E)-H(G)=H(E\mid G).
\label{eq:entropyquotient}
\end{equation}
Hence conditional selection entropy within behavior cells is the exact optimal \emph{expected} coding value of quotienting.
\end{proposition}

\paragraph{Proof of \cref{prop:canonicalrepresentative}.}
The masses $\{\bar\nu(g)\}_g$ sum to one, so $\nu^\dagger$ is a sample-independent prior.  By construction, $-\ln\nu^\dagger(s(g))=-\ln\bar\nu(g)=K_{W,A}(g)$, and $\Pi(s(g))=g$ gives the same population loss behavior.  For the entropy statement, cross-entropy gives
$\E[-\ln\rho(E)]=H(E)+\KL(P_E\|\rho)$ and
$\E[-\ln\sigma(G)]=H(G)+\KL(P_G\|\sigma)$; minimizing at the two marginals and using that $G$ is a function of $E$ yields $H(E)-H(G)=H(E\mid G)$.

\subsection{Exact routing composition and certified feasibility}
\label{app:routedproof}

\begin{proof}[Proof of \cref{prop:routedfactorization}]
For any implementation tuple $e$, the identity $g_e=g$ on $\mathcal X$
holds if and only if $g_{u,e_u}=g|_{A_u}$ for every $u$, since the regions
partition the domain. Therefore its preimage is exactly $\prod_uF_u(g)$.
For countable implementation spaces, nonnegative summation under a product
prior gives
\[
 \nu(F(g))=\sum_{e_1\in F_1}\cdots\sum_{e_U\in F_U}
              \prod_u\nu_u(e_u)
          =\prod_u\sum_{e_u\in F_u}\nu_u(e_u).
\]
Taking negative logarithms proves additivity. Expanding each local mixture
before summing proves the format formula. Local certified subsets give a
subset of this product fiber, so their mass is a lower bound. A component
of a shared-field mixture contributes $\pi(a)\prod_up_u(g)$ to total mass;
other components can add further mass. No assumption of independence of
data, expert losses, or routing frequencies enters this identity. Sample
independence is required when the prior is used for certification.
\end{proof}

\paragraph{A mixture of products retains an explicit mass formula.}
The composition statement also permits correlated expert priors. For the
same fixed partition and loss maps, if
$\nu=\sum_h\lambda_h\bigotimes_u\nu_{u,h}$, then
\[
 \bar\nu(g)=\sum_h\lambda_h\prod_u\nu_{u,h}(F_u(g)).
\]
Thus independence can be weakened to an explicit mixture of products, with
the mixture weights retained. Under an arbitrary joint prior the fiber
still factorizes as a set, and its joint probability must be used.

\paragraph{Feasibility dominance under a common prior.}
For each fixed candidate, $\bar\nu(g_e)\ge\nu(e)$, so its exact-fiber
Occam--KL bound $U_{\rm BCC}(e)$ is no greater than $U_{\rm raw}(e)$ under
the same empirical loss, sample size, and confidence charge. Consequently,
for a common candidate set $\mathcal H$, physical cost $B_{\rm phys}$, and
threshold $\varepsilon$,
\begin{equation}
 \{e\in\mathcal H:U_{\rm raw}(e)\le\varepsilon\}
 \subseteq\{e\in\mathcal H:U_{\rm BCC}(e)\le\varepsilon\},\qquad
 B^*_{\rm BCC}(\varepsilon)\le B^*_{\rm raw}(\varepsilon),
 \label{eq:routedfeasibility}
\end{equation}
where $B^*$ minimizes physical cost over the indicated set and is $+\infty$
when empty. The inclusion holds for any architecture with known fiber mass.
Fixed routing supplies an exact local computation of that mass. The
100-seed study uses the same prior and candidate budgets for both routes,
so this inclusion explains the ordering across its entire tolerance grid.
Its strict gains depend on the measured losses and masses. A transferred
empirical cell additionally pays its transfer and confidence terms.

\paragraph{Applicability to MoE architectures.}
The result applies when each input's loss depends on one edited expert and
fixed shared fields. A router selecting an entire predictor for a document
is one example. Token-routed transformer blocks require checking the loss
dependence of the full computation: later attention can mix tokens assigned
to different experts, and later routers can see states changed by earlier
edits. Switch Transformers use sparse expert selection~\citep{switch2022};
Mixtral selects two experts per token and combines their outputs~\citep{mixtral2024}.
These architectures motivate certification of routing and expert interfaces.
For a routed block with independently editable experts, preserving each
expert's complete output on its fixed reachable region supplies a
composable block transition. Combining such transition cells across layers
gives a certified product subset of the full-model fiber. The present routed study
tests the exact loss-local premise on a fully known finite state domain.

\paragraph{Counterexamples identify the required structure.}
With one scalar output $ab$ and reference $(a,b)=(0,0)$, either individual
edit to one preserves the output, while editing both changes it. Separate
loss equality therefore does not compose in general sequential systems.
Even when a routed fiber is a product set, a correlated prior assigning
mass $1/2$ to each of $(0,0)$ and $(1,1)$ gives the singleton fiber
$\{(0,0)\}$ mass $1/2$, although the product of its marginal masses is
$1/4$. Finally, changing the router changes the regions themselves; a
certificate must fix the router, prove its assignments stable, or include
the changed routing in the certified object.
\subsection{Proof of the behavioral certificate}
\label{proof:occam}
\paragraph{Code-length relaxation.}
If the implementation prior assigns a reconstructed prefix code of length $B(e)$ mass at least $2^{-B(e)}$, then $K_{W,A}(g_e)\le B(e)\ln2$, and Pinsker gives the simpler consequence
\begin{equation}
\Risk(f_e)\le\Rh_S(f_e)+\Gamma_m(B,\Delta,\delta),\quad
\Gamma_m=\Delta\sqrt{\frac{B\ln2+\ln(1/\delta)}{2m}}.
\label{eq:gamma_main}
\end{equation}

For a fixed $[0,1]$-valued behavior $g$, the one-sided Chernoff inequality gives
$\Pr\{R_0(g)>\operatorname{kl}^{-1}_+(\widehat R_0(g),c)\}\le e^{-mc}$.  Set
$c_g=[K_{W,A}(g)+\ln(1/\delta)]/m$ and allocate failure probability
$\delta\bar\nu(g)$ to each countable behavior.  A union bound proves \cref{eq:behavioral_bound}.  If the implementation prior assigns the decoded implementation mass at least $2^{-B}$, then the whole equivalence class has at least that mass and $K_{W,A}\le B\ln2$.  Pinsker's inequality yields \cref{eq:gamma_main}.  All sample-dependent scales, allocation choices, masks, and checkpoints must therefore be encoded.

\paragraph{Proof of \cref{cor:deployment}.}
Pointwise,
$\ell(f_e,X)-\ell(f_0,X)\le[\ell(f_e,X)-\ell(f_0,X)]_+\le\Delta d_e(X)$;
taking expectations and applying \cref{thm:behavioral} proves
\cref{eq:damagecertificate}.  The event is simultaneous, so any cost-constrained
selection made on that event remains covered.  

\paragraph{The shared PAC--Bayes event.}
For a fixed bounded loss of mean $p\in(0,1)$, convexity reduces its KL
exponential moment to the Bernoulli case~\citep{maurer2004note}, giving
\[
\E_S e^{m\operatorname{kl}(\widehat R_0(g)\|R_0(g))}
\le \xi_m:=\sum_{k=0}^m\binom mk
(k/m)^k(1-k/m)^{m-k}\le2\sqrt m,
\]
with $0^0=1$.  The last inequality follows from that reference for $m\ge8$;
substitution in the finite sum verifies $\xi_m^2\le4m$ for $1\le m\le7$.
Means zero and one give a moment of one.  Averaging over $\nu$ and applying
Markov's inequality yields a single event of probability at least $1-\delta$.
On this event, KL convexity and the variational change-of-measure inequality
give, simultaneously for all $Q\ll\nu$,
\[
m\operatorname{kl}(\widehat R_0(Q)\|R_0(Q))
\le\KL(Q\|\nu)+\ln(2\sqrt m/\delta).
\]
For posteriors with infinite KL, use $R_0(Q)\le1$.  Thus the posterior
may be selected on $S$ without an extra posterior-selection union bound.

\label{proof:transfer}
For \cref{thm:publicquotient}, use the PAC--Bayes--KL inequality with the arbitrary data-dependent posterior $Q_S$ to obtain the core term in \cref{eq:posteriortransfer}.  Conditional on $S$, the function
$X\mapsto\E_{e\sim Q_S}|g_e(X)-g_{e_0}(X)|$ is fixed and $[0,1]$-valued.  The one-sided Chernoff bound on the independent probes therefore gives
$\E_X\E_{e\sim Q_S}|g_e(X)-g_{e_0}(X)|\le\eta_Z$.  Finally,
$R_0(g_{e_0})\le R_0(Q_S)+\eta_Z$ and a union bound prove \cref{eq:posteriortransfer}.  For the exact-cell specialization, conditioning the prior gives
$\KL(Q_S\|\nu)=-\ln\mu_S$ and cell membership gives
$\widehat R_0(Q_S)=\widehat R_0(g_{e_0})$, proving \cref{eq:publicquotient}.  Under a product prior and product cell, masses multiply.  Layerwise transition equality on the reference reached states makes the training computations equal inductively; per-layer loss equality alone does not, and the claim fails if a bypass crossing the cut is omitted.

\paragraph{Reading the transfer argument.}
The posterior is an accounting device: conditioning the prior on a computable
cell turns its mass into a KL complexity.  PAC--Bayes controls this posterior's
average risk even though the cell was selected on $S$.  The deployed object,
however, is the chosen implementation $e_0$.  Independent probes certify the
remaining discrepancy needed to pass from the average to that implementation.
Freezing both objects before $Z$ makes this last discrepancy a fixed bounded
function for concentration.  Thus \cref{thm:publicquotient} closes the gap
between a computable empirical cell and a population guarantee for deployment.

\paragraph{Confidence and fit--transfer decomposition.}
When $U_{\rm raw}$ is the interior KL root, the same threshold decomposes exactly as
\begin{equation}
G_\star=\log_2\!\frac{2\sqrt m\,\alpha}{\delta}
+\frac{m}{\ln2}\left[
\operatorname{kl}(q_0\|U_{\rm raw})-
\operatorname{kl}(q_Q\|U_{\rm raw}-\tau)
\right].
\label{eq:quotientcrossoverdecomp}
\end{equation}
The first term is the confidence toll; the second is the exact empirical-fit-and-transfer toll in equivalent bits.  In particular, $\log_2(2\sqrt m)$ is the full threshold only when $q_Q=q_0$, $\tau=0$, and $\alpha=\delta$.

\paragraph{Proof of \cref{prop:quotientcrossover}.}
\label{proof:crossover}
The one-sided inverse is strictly increasing in its radius on the interior.  Therefore
$\tau+\operatorname{kl}^{-1}_+(q_Q,c_Q)<U_{\rm raw}$ is possible only when $U_{\rm raw}-\tau>q_Q$, and in that case it is equivalent to
$c_Q<\operatorname{kl}(q_Q\|U_{\rm raw}-\tau)$.  Substituting
$c_Q=[(B-G)\ln2+\ln(2\sqrt m/\delta)]/m$ and solving for $G$ gives \cref{eq:quotientcrossover}.  If $U_{\rm raw}$ is the interior root, then
$m\operatorname{kl}(q_0\|U_{\rm raw})=B\ln2+\ln(1/\alpha)$; substitution gives \cref{eq:quotientcrossoverdecomp}.
The arithmetic comparison itself does not select a new confidence event.
To return the smaller bound adaptively, cover both events by a union bound;
for overall failure $\varepsilon$, preallocate
$\alpha+\delta+\eta\le\varepsilon$ and recompute both ledgers.
Tables comparing separately valid 95\% bounds do not by themselves certify
their adaptively chosen minimum at 95\%.

\paragraph{Fiber growth and the limiting transfer cost.}
Under the notation of \cref{cor:probebudget}, for a uniform implementation prior with fiber size $F$, $G=\log_2F$.  At an
interior KL root $q_Q<U_{\rm core}<1$, holding the other ledger terms fixed,
\begin{equation}
\frac{\partial\Delta_{\rm core}}{\partial G}
=\frac{\ln2}{m}\frac{U_{\rm core}(1-U_{\rm core})}
{U_{\rm core}-q_Q}>0.
\label{eq:fiberscaling}
\end{equation}
Holding the transfer scenario fixed, larger certified fibers lower its budget,
and as $\Delta_{\rm core}\downarrow0$ the zero-disagreement law is
$n_\star=\Theta(\ln(1/\eta)/\Delta_{\rm core})$.  More generally, if a sequence
of valid transfer bounds satisfies $\tau_n\to d_\infty$, then every sufficiently
large $n$ wins when $d_\infty<\Delta_{\rm core}$ and no sufficiently large $n$ wins when
$d_\infty>\Delta_{\rm core}$.
If $\Delta_{\rm core}\le0$, no number of zero-disagreement probes can make
that ledger beat the direct bound; more probes can remove transfer cost but
cannot repair an unfavorable confidence-and-complexity core.

\paragraph{Proof of \cref{cor:probebudget}.}
\label{proof:probebudget}
The transferred bound is $U_{\rm core}+\tau_n$, so strict improvement over
$U_{\rm raw}$ is equivalent to $\tau_n<U_{\rm raw}-U_{\rm core}$.  With zero empirical disagreement,
$\operatorname{kl}(0\|r)=-\ln(1-r)$ gives
$\tau_n=1-\exp[-\ln(1/\eta)/n]=1-\eta^{1/n}$.  Solving the strict inequality
$1-\eta^{1/n}<\Delta_{\rm core}$ for the least integer $n$ gives
\cref{eq:zerodisagreementbudget}.  Since
$\partial_u\operatorname{kl}(q_Q\|u)=(u-q_Q)/[u(1-u)]$, implicit
differentiation of
$\operatorname{kl}(q_Q\|U_{\rm core})=[(B-G)\ln2+C]/m$ gives
\cref{eq:fiberscaling}.  The elementary inequalities
$x\le-\ln(1-x)\le x/(1-x)$ for $x\in(0,1)$ give the stated small-gap
scaling.  The limit claim follows directly from $\tau_n\to d_\infty$.
When $\Delta_{\rm core}\le0$, the nonnegative transfer term cannot yield a
strict win.  Equality of $d_\infty$ and $\Delta_{\rm core}$ alone does not
determine the asymptotic strict ordering.

\paragraph{Interpreting the probe budget.}
There is no width-only probe law.  If a width-$w$ universe contains
$F(w)$ genuinely equivalent implementations, its saving is exactly
$\log_2F(w)$ bits: exponential fiber growth gives a saving linear in the
number of redundant coded choices, whereas polynomial growth gives only
$O(\log w)$ bits.  Quotienting is never worthwhile for the stated comparator
when its saving fails to overcome the confidence and empirical-fit tolls
($\Delta_{\rm core}\le0$), or, on the transfer-coverage event, when the true
population disagreement is at least $\Delta_{\rm core}$.  It can be valid
yet impractical at any realistic probe budget.  The corollary therefore serves
as a prospective resource-allocation rule: spend validation effort only when
the available core improvement can cover the declared transfer scenario.

\paragraph{Prospective planning versus sequential validation.}
The zero-disagreement budget is conditional on that declared scenario; it
does not predict that future probes will agree.  At a fixed $n$, the observed
disagreement determines the certificate.  Repeatedly checking probes and
stopping at the first favorable fixed-$\eta$ bound is not justified by
\cref{thm:publicquotient}.  One valid sequential version preallocates
$\eta_n>0$ with $\sum_{n\ge1}\eta_n\le\eta$, for example
$\eta_n=\eta/[n(n+1)]$, and uses
$\tau_n=\operatorname{kl}^{-1}_+(\widehat d_{Z_{1:n}},\ln(1/\eta_n)/n)$.
Conditionally on $S$, a union bound then covers all inspected $n$; the core
event gives total failure at most $\delta+\eta$.  The stopping threshold must
be recomputed with these charges.  The reported fixed-budget audits use
predeclared probe counts, so their certificates require no such adjustment.

\subsection{Computable cells and the margin certificate}
\label{app:quotient}

The exact population quotient in \cref{eq:behavioral_mass} is generally
difficult to compute.  Certification only needs a valid upper bound on
complexity: a verified subset of a population fiber lower-bounds its mass and
therefore upper-bounds its negative logarithm.  The following constructions
make this route computable under explicit hypotheses.

\paragraph{Architecture determines the composition rule.}
The certificate indexes complete population behaviors. For fixed routed
losses satisfying \cref{prop:routedfactorization}, local fibers compose
exactly, so the global mass follows from local masses. For a sequential
network, preserving every complete transition on reachable states gives
a sufficient composition rule. Residual, attention, normalization, scale,
and cache paths must be included. The controlled complete-decoder and
GPT-2 cache studies certify such transitions
(\cref{core:complete,core:cache}); the FFN and head audits
use their stated local universes. The constructions below give computable
fiber subsets without recovering the entire equivalence class.

\begin{proposition}[A computable output-head quotient]
\label{prop:headquotient}
Fix an activation format and a population set $\mathcal X$.  Let $\phi(x)\in\mathbb R^d$ be the complete fixed representation reaching a $C$-class quantized linear head, and let each of its $N=Cd$ symbols lie in a numerical alphabet $\mathcal A$ of size $q\ge2$ and diameter $D_{\mathcal A}$, including fixed scales.  Let $r\in\{0,\ldots,N\}$.  For a reference code $W$, write $\hat y(x)=\arg\max_c\langle W_c,\phi(x)\rangle$ and
\[
m_k(x)=\langle W_{\hat y(x)}-W_k,\phi(x)\rangle,
\qquad k\ne\hat y(x).
\]
For each $(x,k)$, assign sensitivity $s_{c,j}(x,k)=D_{\mathcal A}|\phi_j(x)|$ when $c\in\{\hat y(x),k\}$ and zero otherwise, and let $S_r(x,k)$ be the sum of its $r$ largest coordinate sensitivities.  If
\begin{equation}
\inf_{x\in\mathcal X}\min_{k\ne\hat y(x)}[m_k(x)-S_r(x,k)]>0,
\label{eq:head_margin}
\end{equation}
then every head code within Hamming radius $r$ has the same predictions, hence the same 0--1 loss behavior, on $\mathcal X$.  Under an independent uniform head prior, with $K_{\rm upstream}$ charging the fixed upstream implementation and all other metadata,
\begin{equation}
K_{W,A}(g)\le K_{\rm upstream}+N\ln q-
\ln\sum_{j=0}^{r}\binom Nj(q-1)^j.
\label{eq:hamming_cell}
\end{equation}
\end{proposition}

For a finite declared population, \cref{eq:head_margin} is evaluated by enumeration; for an infinite population, any sound neural-network verifier may lower-bound the margins and upper-bound the sensitivities.  The representation $\phi$ is recomputed for A4 and A8, so activation precision can change the certified radius.  This is a sufficient construction rather than an assumption that a Hamming ball happens to be equivalent.  Its proof is immediate: changing at most $r$ symbols decreases any winning logit margin by at most $S_r$.

\paragraph{Proof and half-margin specialization of \cref{prop:coordinatecell}.}
\label{app:coordinateproof}
\begin{proof}
For every allowed simultaneous edit, row $k$ changes its logit by at most
$\alpha L_k$. Its adverse change relative to winner $w$ is therefore at
most $\alpha(L_w+L_k)$. The strict inequality in
\cref{prop:coordinatecell} preserves a positive winning gap. When both
budgets are zero, both logits and their tie order stay unchanged.
All predictions are preserved. The editable coordinates are disjoint,
giving $q^{D_*}$ complete head codes. Their uniform conditional prior mass
is $q^{D_*-N}$; including upstream and metadata cost proves the bound.
\end{proof}

The rowwise half-margin construction used in the experiments is a
sufficient specialization. Let $b_k$ be the infimum of the unscaled
winner--runner-up margin when $k$ wins and the winner--$k$ margin otherwise.
Choosing $L_k=0$ or $0<L_k<b_k/2$ implies the pairwise condition: at any
positive gap $m_{wk}$, both $b_w\le m_{wk}$ and $b_k\le m_{wk}$, while a
zero gap forces both budgets to zero. Joint budgets can permit more edits:
for a unit gap, $(L_w,L_k)=(0.7,0.2)$ satisfies the pairwise test although
the first row exceeds half the margin. This strengthens the sufficient
condition; the reported experiment counts retain their original budgets.

This cell can be large even when no radius-one ball lies in the fiber,
since it can exclude sensitive directions. The population bounds on
margins and feature magnitudes must be sound. The head experiment evaluates
them over its complete declared domain (\cref{core:head}).

\paragraph{Reachable transitions and finite cuts.}
At layer $\ell$, let $\Omega_\ell$ be a sample-independent, certified finite set of states that can reach the quantized boundary, and let
$T_\ell(\theta_\ell)=(F_{\ell,\theta_\ell}(u):u\in\Omega_\ell)$.  Use a sample-independent code that indexes each distinct transition tuple uniformly.  If all paths crossing the boundary are included, then
\begin{equation}
K_{W,A}(g)\le\sum_\ell\ln|\{T_\ell(\theta_\ell):\theta_\ell\in\Theta_\ell\}|+B_{\rm downstream}\ln2.
\label{eq:transition_code}
\end{equation}
Here $K_{W,A}$ refers to this explicitly chosen tuple-code prior, not necessarily
the original implementation prior.  The reachable-state sets must cover every
admissible upstream choice.  A full-precision residual, dynamic scale,
normalization state, attention probability, or KV-cache path crossing the cut
must be included.  Thus activation precision enters the code geometry.

\subsection{Behavioral mass and perturbation risk}
\label{app:perturbationcoverage}

\paragraph{The question answered by this appendix.}
Nominal damage compares a deployed implementation with a useful reference.
It does not specify how that implementation reacts to a subsequent change in
its weight symbols or quantizer thresholds.  Behavioral multiplicity can
address this second question when mass is measured under a declared law of
implementation changes.  The relevant object is the probability of preserving
behavior under that law.  A large coding-prior fiber alone does not establish
a worst-case margin, tolerance to arbitrary hardware faults, or robustness to
another input distribution.

\subsubsection{Coverage and re-centering}

Let $\kappa_e$ be a specified distribution over complete perturbed
implementations and write $R_{\kappa}(e)=\mathbb E_{e'\sim\kappa_e}R(e')$.
The kernel describes implementation variation; it need not equal the coding
prior $\nu$, and an adaptively centered kernel is not automatically a legal
sample-independent coding prior.  For the exact population fiber
$F_e=\{e':g_{e'}=g_e\}$, define
\[
 s_e=\kappa_e(F_e),\qquad K_{\kappa}(g_e)=-\ln s_e.
\]
Thus the same push-forward construction has an operational interpretation:
$s_e$ is the probability that a perturbation preserves the chosen loss
function.  It is not the probability of identical logits unless the verifier
establishes that stronger equivalence.

\begin{proposition}[Covered perturbation risk]
\label{prop:coveredrisk}
Let $g_e\in[0,1]$, let $\kappa_e$ be a declared perturbation law over complete
implementations centered at $e$, and write
$R_\kappa(e)=\mathbb{E}_{e'\sim\kappa_e}R(e')$. Let $C$ be a
construction-selected cell with known kernel mass $s=\kappa_e(C)$ and, when
$s>0$, set $Q=\kappa_e(\cdot\mid C)$. If simultaneous valid bounds give
$R(e)\le U$ and $\mathbb{E}_{X,\,e'\sim Q}\lvert g_{e'}(X)-g_e(X)\rvert\le\tau$,
then
\begin{equation}
  R_\kappa(e)\le 1-s+s\min\{1,U+\tau\}.
  \label{eq:coveredrisk}
\end{equation}
For $s=0$ the bound is one and $Q$ need not be defined. Exact cells have
$\tau=0$; empirical cells obtain $\tau$ from independent probes
(\cref{thm:publicquotient}). For the full population fiber $F_e$, $\tau=0$ and
$s=\kappa_e(F_e)=e^{-K_\kappa(g_e)}$, so
$R_\kappa(e)\le 1-e^{-K_\kappa(g_e)}(1-U)$.
\end{proposition}
\begin{proof}
Split the kernel expectation over $C$ and its complement.  The complement
contributes at most $1-s$.  On $C$, the pointwise inequality
$g_{e'}(X)\le g_e(X)+|g_{e'}(X)-g_e(X)|$ gives conditional risk at most
$R(e)+\tau$, and boundedness also gives an upper bound of one.  Substitute
$R(e)\le U$.  For the exact fiber, conditional discrepancy is zero and
$s=e^{-K_{\kappa}}$.  No independence between layers or geometric regularity
of the fiber is required.
\end{proof}
The guarantee therefore, depends on three quantities: nominal risk $U$, retained
mass $s$, and within-cell discrepancy $\tau$. All members of a fiber share the
same nominal loss, so deploying the member whose kernel retains the most mass
tightens the bound at no cost in accuracy (re-centering).
For empirical cells, freeze $C$ and $Q$ before probes and independently draw
$X_i\sim\mathcal D$ and $E_i\sim Q$.  The observations
$|g_{E_i}(X_i)-g_e(X_i)|\in[0,1]$ have the conditional discrepancy as their
mean.  One-sided KL inversion and a union bound over the frozen candidates
give $\tau$.  Nominal bounds can use the earlier BCC or direct-code route.
This applies the independent-validation principle of
\cref{thm:publicquotient} to a specified perturbation law.

\begin{corollary}[Required coverage]
\label{cor:requiredcoverage}
Put $v=\min\{1,U+\tau\}$ and fix a risk threshold $\varepsilon<1$.
If $v>\varepsilon$, no $s\le1$ makes \cref{eq:coveredrisk} pass the threshold.
If $v\le\varepsilon$, it passes exactly when
\[
s\ge\frac{1-\varepsilon}{1-v},
\quad\text{equivalently}\quad
-\ln s\le\ln\frac{1-v}{1-\varepsilon}.
\]
\end{corollary}
\begin{proof}
Rearrange $1-s(1-v)\le\varepsilon$; here $1-v>0$.  Taking logarithms
gives the equivalent complexity condition.  This is a conditional rule for
given $U$ and $\tau$; changing the cell can change $\tau$.
\end{proof}

The mixture bound is no larger than
$\min\{1,U+1-s+s\tau\}$, which separately charges the full outside-cell
damage.  It need not beat direct certification of $R_\kappa$ from perturbed
predictions. A minimum of the covered and direct routes is valid when their
confidence events share a common ledger.

\subsubsection{Damage audits and construction costs}
\label{app:residualtheory}
The behavioral certificate supplies an exactly integrable zero-damage
stratum and its exact probability mass from a proof, with zero perturbation
queries. The population-margin construction uses nominal activations and
margins, whose acquisition and processing carry a separate cost. Conditional
sampling then spends perturbation queries on the complement and reweights
by its known mass, an application of classical stratified
Monte Carlo~\citep{owen2013mc}.

\begin{proposition}[Mass-guided damage auditing]
\label{prop:residualaudit}
Fix an implementation, a perturbation law $\kappa$, and a certified cell $C$
with known mass $s=\kappa(C)$. Let $Z(E,X)\in[0,1]$ be identically zero on
$C$ for every population input, and let $w=1-s>0$. Write
$Y\sim Z(E,X)\mid E\notin C$, with mean $\mu$ and variance $\sigma^2$,
where inputs and implementations are independent. For $n$ iid queries,
the direct estimator $\widehat D_{\rm dir}=n^{-1}\sum_i Z_i$ and
$\widehat D_C=w n^{-1}\sum_iY_i$ are both unbiased for $D=\mathbb E Z$,
and
\begin{align}
\operatorname{Var}(\widehat D_{\rm dir})
 &=\frac{w\sigma^2+w(1-w)\mu^2}{n},
 &\operatorname{Var}(\widehat D_C)&=\frac{w^2\sigma^2}{n}
 \le w\operatorname{Var}(\widehat D_{\rm dir}).
\label{eq:residualvariance}
\end{align}
Any valid upper bound $u$ on $\mu$ gives $D\le wu$ at the same confidence
level. If $s=1$, then $D=0$ without queries.
\end{proposition}
\begin{proof}
The zero contribution on $C$ gives $D=w\mu$ and
$\mathbb E Z^2=w(\sigma^2+\mu^2)$. Subtracting $(w\mu)^2$ and dividing
by $n$ gives the direct variance; iid conditional sampling gives the second.
Their difference is $w(1-w)(\sigma^2+\mu^2)/n\ge0$; the stronger displayed
factor follows because
$w\operatorname{Var}(\widehat D_{\rm dir})-
\operatorname{Var}(\widehat D_C)=w^2(1-w)\mu^2/n\ge0$.
Multiplying a valid mean upper bound by the known nonnegative $w$ preserves
its confidence event.
\end{proof}

For prediction disagreement, a prediction-fiber subset supplies $C$.
For positive 0--1 damage, a loss-fiber subset suffices. If the nominal
correctness probability $c$ is known, both routes can also sample inputs
conditional on nominal correctness: their estimators become
$c\overline Z$ and $cw\overline Y$, respectively, and both variances are
multiplied by $c^2$. Our comparator uses this improvement. For Bernoulli
observations, we use the exact one-sided binomial upper limit for $u$.
The procedure requires a certified cell, its mass, and a sampler for its
complement. Query savings reuse that certificate; constructing it has a
separate cost. The variance comparison is an expectation over audits and
is not a pointwise ordering of every realized confidence bound.

\begin{corollary}[Prospective query budget and construction-cost crossover]
\label{cor:auditplanning}
Under \cref{prop:residualaudit}, let $0<s<1$, $w=1-s$, and let $N\ge1$
be an integer direct-audit budget. The conditional estimator with
$n=\lceil wN\rceil$ queries has MSE no larger than the direct estimator
with $N$ queries. If cell construction and sampler setup cost $b\ge0$
direct-query equivalents, and each conditional query costs $r>0$ such
equivalents, this allocation costs less whenever
\begin{equation}
b+r\lceil(1-s)N\rceil<N.
\label{eq:auditcrossover}
\end{equation}
For $r=1$, the smallest integer budget passing the test is \mbox{$N_{\min}=\lceil(\lfloor b\rfloor+1)/s\rceil$}.
Ignoring rounding gives $N>b/s$. More generally, when $r(1-s)<1$,
the unrounded threshold is $N>b/[1-r(1-s)]$.
\end{corollary}
\begin{proof}
Put $V=w\sigma^2+w(1-w)\mu^2$, the direct one-query variance.
By \cref{eq:residualvariance}, $w^2\sigma^2\le wV$. Since both estimators
are unbiased and $n\ge wN$,
\[
\operatorname{MSE}(\widehat D_C;n)
 =\frac{w^2\sigma^2}{n}
 \le\frac{wV}{n}\le\frac{V}{N}
 =\operatorname{MSE}(\widehat D_{\rm dir};N).
\]
The cost is $b+rn$, giving \cref{eq:auditcrossover}. For $r=1$,
$N-\lceil(1-s)N\rceil=\lfloor sN\rfloor$, so the test is equivalent to
$\lfloor sN\rfloor>b$, or $sN\ge\lfloor b\rfloor+1$. This gives
$N_{\min}$. Dropping the ceiling gives the unrounded thresholds.
\end{proof}

Thus $s$ supplies a guaranteed variance-reduction factor of at least
$1/(1-s)$ before audit outcomes are observed. For Bernoulli conditional
outcomes with $0<\mu<1$, the exact factor is
$(1-w\mu)/[w(1-\mu)]$; it approaches $1/w$ as $\mu\downarrow0$, so this
mass-only guarantee is sharp. Reusing a cell over several audits shares
$b$ across them: the total-cost test is
$b+r\sum_j\lceil wN_j\rceil<\sum_jN_j$.
For $s=1$, the risk is exactly zero and no perturbation queries are needed;
$s=0$ provides no mass-based query saving. These rules compare expected
squared error. Equal confidence-bound width and elapsed-time savings require
their own cost and precision analysis.

\subsubsection{One mass for coding and stability}
\label{app:massstability}

\paragraph{Invariance under behavior-preserving recoding.}
\label{app:massinvariance}
Let $T:\mathcal E\to\mathcal E'$ map countable implementation spaces and
satisfy $g'_{T(e)}=g_e$ for every $e$. Transport the prior and each declared
kernel by $\nu'=T_\#\nu$ and $\kappa'_j=T_\#\kappa_j$, where
$T_\#\mu(A)=\mu(T^{-1}(A))$. For the fibers
$F_g=\{e:g_e=g\}$ and $F'_g=\{e':g'_{e'}=g\}$,
\begin{equation}
\bar\nu'(g)=\bar\nu(g),\qquad
\kappa'_j(F'_g)=\kappa_j(F_g).
\label{eq:massinvariance}
\end{equation}
Indeed, $T^{-1}(F'_g)=F_g$; applying the push-forward definition proves both
equalities. Also $\mathbb E_{e'\sim\kappa'_j}R'(e')=
\mathbb E_{e\sim\kappa_j}R(e)$ by loss preservation. Thus behavioral
complexity, kernel coverage, actual perturbation risk, and the bounds in
\cref{prop:massstability} are unchanged when mixture weights are retained.
Injectivity, differentiability, a volume element, and a group structure are
unnecessary. Statistical use retains the sample-independence condition on
the prior and any charged kernel dictionary.

For example, three equally probable implementations with behaviors
$(g_A,g_A,g_B)$ give fiber masses $(2/3,1/3)$. Merging the first two
implementations and adding their masses preserves those probabilities and
the complexity $\ln(3/2)$ of $g_A$. Assigning a new uniform prior to the two
resulting codes gives masses $(1/2,1/2)$ and complexity $\ln2$ instead.
Likewise, a coordinate-defined edit or noise law must be transported to
represent the same perturbation experiment. The invariant object consists
of the loss map together with its declared measures.

The coding prior and perturbation kernel need not be identical.
\Cref{prop:massstability} allows a declared dictionary of kernels and
charges for choosing one, including its center.

\paragraph{Proof of \cref{prop:massstability}.}

\begin{proof}
The prior gives $\nu(F_e)\ge\pi_j\kappa_j(F_e)=\pi_js_j$; taking negative
logarithms gives \cref{eq:kernelcode}. Positive damage is zero on $F_e$.
Outside it, $[g_{e'}(X)-g_e(X)]_+\le1-g_e(X)$, because both losses lie
in $[0,1]$. Independence of $e'$ and $X$ therefore bounds the complement
contribution by $(1-s_j)(1-R(e))$. Finally
$R_{\kappa_j}(e)\le R(e)+D_j^+(e)\le1-s_j+s_jR(e)$, and $s_j\ge0$
allows substitution of $U$. The first inequality is simultaneous over
implementations and dictionary indices; using it in a statistical certificate
still requires the sample-independent prior and confidence event of
\cref{thm:behavioral}. No zero-cost adaptive prior is introduced.
\end{proof}

Thus one mass has two meanings: it is the probability of preserving the
entire chosen loss function, and, after charging the kernel index, it bounds
the code assigned to that function. Increasing $s_j$ tightens both upper
bounds for fixed $L_j$ and nominal risk. Actual damage also depends on
losses outside the fiber. \Cref{eq:massinvariance} preserves these code and
risk statements under behavior-preserving recoding with transported measures.

\paragraph{Using a certified subset.}
A complete fiber need not be recovered. If $C\subseteq F_e$ has a certified
mass lower bound $0<\underline s_j\le\kappa_j(C)$ and $R(e)\le U$ with $U\in[0,1]$,
then on their joint validity event the same construction gives
\[
K_{W,A}(g_e)\le L_j-\ln\underline s_j,\qquad
R_{\kappa_j}(e)\le1-\underline s_j+\underline s_j U.
\]
Indeed, $s_j\ge\underline s_j$, and both right-hand sides are nonincreasing
in the substituted mass. A margin-certified coordinate bank therefore
supplies both bounds without enumerating every allowed fault combination.
Under a kernel confined to a verified prediction-preserving cell,
$s_j=1$ and every permitted implementation retains the reference
prediction-error risk. A data-selected bank must still use a charged
sample-independent kernel dictionary or another valid prior construction
for the coding statement.

The risk bound is sharp given only $s_j$ and $R(e)=u<1$: on a singleton
input domain take reference loss $u$, give its fiber mass $s_j$, and assign
loss one to all remaining mass. Then
$R_{\kappa_j}(e)=s_j u+(1-s_j)$ and
$D_j^+(e)=(1-s_j)(1-u)$. A stronger uniform bound therefore requires
information about losses outside the fiber.

\clearpage
\section{Experimental protocols, principal ablations, and controls}
\label{core:protocols}
\long\def\scaleonesummarytable{%
\begin{table}[H]
\caption{Scale-I exact-population summary.  Weights is the implementation
alphabet; A4 fiber and A8 fiber count minimum-error implementations with the
displayed population behavior; Joint $K$ is the negative base-two logarithm
of their combined push-forward mass under equal A4/A8 format priors; and A8
lower error counts seeds, out of 50 unselected exhaustive replications, in
which A8 has lower minimum population error.}
\label{tab:scaleone_summary}
\centering
\small
\begin{tabular}{lrrrr}
\toprule
Weights & A4 fiber & A8 fiber & Joint $K$ (bits) & A8 lower error\\
\midrule
Binary  & 2 & 1 & 8.415  & 29/50\\
\bccrowsep
Ternary & 6 & 4 & 11.943 & 37/50\\
\bottomrule
\end{tabular}
\end{table}
}

\long\def\scaletwosummarytable{%
\begin{table}[H]
\caption{Scale-II certificate summary.  Audit identifies the fixed-center or
data-selected protocol; $n$ is the number of independent transfer probes;
Formats gives the evaluated W/A rows; $G$ is behavioral saving in bits;
$U_{\rm raw}$ and $U_{\rm BCC}$ are direct and transferred quotient bounds;
and Wins counts strict improvements.  Intervals span the stated rows or seeds,
and Gain always means $U_{\rm raw}-U_{\rm BCC}$.}
\label{tab:scaletwo}
\centering
\scriptsize
\setlength{\tabcolsep}{3.6pt}
\begin{tabular}{llrrrr}
\toprule
Audit & Formats & $n$ & $G$ (bits) & $U_{\rm raw}\rightarrow U_{\rm BCC}$ & Wins\\
\midrule
Selected exact & T/A8 & 20,000 & 12.000 & 0.02868 $\rightarrow$ 0.02035--0.02060 & 3/3\\
\bccrowsep
Selected Brier & T/A8 & 20,000 & 12.000 & 0.02890--0.03022 $\rightarrow$ 0.02065--0.02215 & 3/3\\
\bccrowsep
Fixed center & T/A8 & 200,000 & 12.000 & 0.76625980 $\rightarrow$ 0.76625510--0.76625744 & 5/5\\
\bccrowsep
Fixed center & W1/W1.58 $\times$ A4/A8 & 4,000 & 0--12.000 & 0.67842--0.76687 $\rightarrow$ 0.67952--0.76777 & 0/4\\
\bottomrule
\end{tabular}
\end{table}
}

\long\def\scalethreesummarytable{%
\begin{table}[H]
\caption{Scale-III GPT-2-small quotient summary.  Format gives weight/activation
precision; $B_{\rm full}$ is raw complexity of both final-FFN matrices in
Mbit; $M_S/|\mathcal H|$ is construction multiplicity over local-universe size;
$G=\log_2 M_S$ is quotient saving in bits; $\eta_Z$ is independent transfer
cost; and $U_{\rm raw}^{\rm loc}\rightarrow U_{\rm BCC}^{\rm loc}$ gives the
direct and transferred quotient per-format 95\% local-prior certificates. Both full-FFN bounds equal one.}
\label{tab:scalethree}\label{tab:gpt2quotient}
\centering
\scriptsize
\begin{tabular}{lrrrrr}
\toprule
Format & $B_{\rm full}$ (Mbit) & $M_S/|\mathcal H|$ & $G$ (bits) & $\eta_Z$ &
$U_{\rm raw}^{\rm loc}\rightarrow U_{\rm BCC}^{\rm loc}$\\
\midrule
W1/A4 & 4.719 & $12{,}071/24{,}577$ & 13.559 & 0.00595 & 0.08062 $\rightarrow$ 0.07256\\
\bccrowsep
W1/A8 & 4.719 & $15{,}611/24{,}577$ & 13.930 & 0.00601 & 0.07851 $\rightarrow$ 0.06988\\
\bccrowsep
W1.58/A4 & 7.479 & $28{,}917/49{,}153$ & 14.820 & 0.00594 & 0.08094 $\rightarrow$ 0.07058\\
\bccrowsep
W1.58/A8 & 7.479 & $35{,}518/49{,}153$ & 15.116 & 0.00564 & 0.07984 $\rightarrow$ 0.06863\\
\bottomrule
\end{tabular}
\end{table}
}

\long\def\practitionerrankingsummarytable{%
\begin{table}[H]
\caption{Prospective GPT-2 practitioner ranking.  Each score is fixed before
this holdout and lower means safer; $\rho_{\Delta\mathrm{NLL}}$ and $\rho_d$
are pooled within-family Spearman correlations with previously unused held-out
NLL increase and positive clipped-NLL damage; Top-3 gives binary/ternary
overlap with the three least-damaging NLL layers.  Higher positive correlations
and overlap are better.  All 24 branches enter each pooled correlation.}
\label{tab:practitionerranking}
\centering
\small
\begin{tabular}{lrrc}
\toprule
Score & $\rho_{\Delta\mathrm{NLL}}$ & $\rho_d$ & Top-3 (B/T)\\
\midrule
Positive-damage BCC & \textbf{0.867} & \textbf{0.930} & 3/1\\
\bccrowsep
Teacher-JS BCC & 0.864 & 0.892 & 2/1\\
\bccrowsep
Original BCC & 0.591 & 0.465 & 0/1\\
\bccrowsep
HAWQ-V2 score & 0.594 & 0.668 & 2/2\\
\bottomrule
\end{tabular}
\end{table}
}

The studies distinguish three objects: exact equality on a declared population,
construction-sample equality validated by independent probes, and empirical
screening followed by evaluation of the assembled model. The loss and sampling
unit below specify each guarantee. 

\subsection{Exact merging across weight and activation precision}
\label{core:exact}
The complete population is the Cartesian cube of
$(-1,-0.55,-0.1,0.35,0.9)$ in float32 arithmetic. For a fixed bias $b$ and
three-class decoder $D$, the network predicts
$\hat y_W(x)=\arg\max_c[D Q_A(\tanh(Wx+b))]_c$, where $W$ is $3\times3$.
The symmetric per-vector quantizer has levels seven for A4 and 127 for A8.
We enumerate all 512 binary and 19,683 ternary codes and compare complete
125-entry loss vectors $g_e(x,y)=\mathbf1\{\hat y_e(x)\ne y\}$.
Equality of these vectors preserves correctness at every input. Prediction
agreement is checked separately, since different wrong classes can have the
same 0--1 loss. Neither condition alone preserves NLL.

The within-alphabet examples in Table~1 use disclosed seeds 110 and 49,
with labels supplied by binary code 502 and ternary code 741 at A8,
respectively. Cross-weight
comparisons hold the task, decoder, bias, and labels fixed. In the seed-110
task, the matrices
\[
W_{\rm B}=\begin{pmatrix}-1&1&1\\-1&1&1\\1&1&1\end{pmatrix},\qquad
W_{\rm T}=\begin{pmatrix}-1&0&0\\-1&1&1\\1&1&1\end{pmatrix}
\]
have identical predictions at A8 on all 125 inputs; the binary matrix also
matches at A4. All are error-free, with class counts $(2,74,49)$.
The A8 zero-loss fiber contains one binary and two ternary codes. Equal
alphabet priors give mass $1/(2\cdot512)+2/(2\cdot19{,}683)$, or 9.926832
bits of behavioral complexity, compared with direct costs 10.000 and
15.264663 bits including the alphabet tag. The 36.907\% reduction in
weight cost compares ideal code lengths $9$ and $9\log_2 3$; fixed-length
blocks require 9 and 15 bits before common metadata.
\paragraph{Behavioral aggregation after lossless compression.} \label{app:compressioncomparison} Compression bounds reward short descriptions of individual predictors~\citep{lotfi2024}. We test whether shared behavior supplies additional savings after replacing the uniform weight code by a lossless entropy code. The comparison uses both existing tasks, seeds 110 and 49, and all four binary/ternary--A4/A8 formats. Their zero-loss fibers contain seven and ten complete implementations, respectively. Bias, decoder, and other inference rules remain public fixed side information; the comparison charges the conditional first-layer code and the two-bit W/A format tag. For an alphabet of size $q$ and a nine-symbol weight code $w$ with symbol counts $n_1,\ldots,n_q$, define the task-independent source distribution \[ \rho_q(w)=\left[\binom{9+q-1}{q-1} \frac{9!}{\prod_{j=1}^q n_j!}\right]^{-1}. \] This chooses a histogram uniformly and then an arrangement uniformly. A deterministic Huffman tree for $\rho_q$, with ties broken by code/node index, assigns length $L_q(w)$. The actual prefix lengths define $\nu_H(w,A,q)=2^{-2-L_q(w)}$; their Kraft sums are exactly one within each alphabet. The tree is determined by $q$, the length nine, and the fixed enumeration, with no fitted histogram or task-dependent codebook. Every code round-trips to its original weights. Direct certification uses the largest individual prior mass in the zero-loss fiber, and BCC uses the sum over that same fiber. Thus the direct comparator may choose the cheapest compressed realization of the shared behavior. 
\begin{table}[H] 
\caption{Behavioral aggregation with a compressed-code prior on the two declared exact tasks. $K_{\rm dir}^{\star}$ is the cheapest direct charge within the zero-loss fiber; $K_{\rm BCC}$ pools that fiber under the same prior. All complexities are in bits and include the W/A format tag. Uniform-code charges are ideal lengths; Huffman direct charges are actual prefix lengths. Bounds concern the bounded 0--1 loss, with the joint confidence ledger stated in the text.} 
\label{tab:compressioncomparison} 
\centering\small \begin{tabular}{l l r r r r} 
\toprule Task seed & Coding prior & $K_{\rm dir}^{\star}$ & $K_{\rm BCC}$ & $U_{\rm dir}^{\star}$ & $U_{\rm BCC}$\\ \midrule 110 & Uniform & 11.000 & 9.366 & 0.02318 & 0.02101\\ \bccrowsep 110 & Huffman & 10.000 & 8.656 & 0.02185 & 0.02007\\ \bccrowsep 49 & Uniform & 16.265 & 12.943 & 0.03012 & 0.02574\\ \bccrowsep 49 & Huffman & 13.000 & 11.046 & 0.02582 & 0.02324\\ \bottomrule \end{tabular} 
\end{table} 
The Huffman code reduces the best direct charge from 11 to 10 bits in seed 110 and from 16.265 to 13 bits in seed 49. Aggregation supplies a further 1.344 and 1.954 bits, tightening the compressed-code bounds in both tasks. The fiber membership and decoded predictions are unchanged; the numerical bounds change because the new code assigns different prior masses. The displayed bounds use $m=512$ fresh iid draws with replacement from each task's uniform 125-input population. Every member of the reported fiber has zero loss on every population input, so its empirical loss is zero for every possible sample. Assigning $\delta=0.05/4=1/80$ to each task/prior event gives joint coverage at least 95\%; the zero-loss Occam--KL inverse is $U(\mu)=1-(\mu/80)^{1/512}$. Exact population membership requires no transfer term. Since the population error is already known to be zero, these bounds illustrate the additional coding benefit of aggregation; they provide no new estimate of that finite risk or whole-model language-model guarantee.
\subsection{Shared behavior across W/A/K/V formats}
\label{app:format_only}

We test whether different weight, activation, and cache formats can
share behavior while the source checkpoint remains fixed. This
isolates the complexity saving obtained by pooling format choices.

\paragraph{Construction and protocol.}
We use three existing four-layer character decoders with width 16,
two attention heads, and a 65-character vocabulary. For each
checkpoint, we evaluate all 256 tuples
\[
(W,A,K,V)\in\{4,8,12,16\}^{4}.
\]
The learned source parameters remain fixed. Quantized values and scales
are determined by the chosen formats and a common deterministic
quantizer. W applies to all linear weight matrices, including the
output head; A applies to linear inputs; K and V apply to the cached
keys and values in every attention layer. Embeddings remain FP16,
while normalization, residual operations, and arithmetic remain FP32.
Sixteen-bit fields use FP16 rounding; lower precisions use symmetric
integer rounding with per-row weight scales, per-token activation
scales, and per-token/head cache scales. All scales are FP32.
The formats are evaluated through numerical emulation, including
the twelve-bit settings.

The declared behavior is the four-token greedy continuation following
a 28-character prefix. We draw 256 construction prefixes uniformly
with replacement from the fixed Shakespeare text. For each checkpoint,
the primary family contains every format tuple whose continuation
matches W16/A16/K16/V16 on every construction prefix. We freeze these
families before evaluating 8,192 independent prefixes from the same
finite text population.

\paragraph{Prior mass and complexity.}
Conditioned on the fixed source checkpoint, the prior is uniform
over the 256 format tuples, so an individual format choice costs
eight bits. A family of $m$ matching implementations contributes
mass $m/256$ and supplies
\[
G=\log_2 m,
\qquad
\overline K_S=-\log_2(m/256)=8-G.
\]
Here, $\overline K_S$ is a conditional complexity upper bound in
bits for the loss behavior restricted to the construction domain.
Continuation equality preserves any declared loss determined by
the generated tokens and target, so the matching implementations
belong to the same loss fiber on that domain. Other implementations
may contribute additional mass.

\begin{table}[t]
\centering
\small
\setlength{\tabcolsep}{4pt}
\caption{Behavioral aggregation across W/A/K/V formats with each
source checkpoint fixed. Every reported family preserves all
four-token continuations on the 256 construction prefixes.
$\overline K_S$ and $G$ are conditional charges in bits.
The last column bounds the probability that any family member
changes any continuation token on a fresh prefix, with joint
95\% confidence across the three rows.}
\label{tab:format_only}
\begin{tabular}{@{}rrrrrrr@{}}
\toprule
Seed & Formats $m$ & Mass $m/256$ &
$\overline K_S$ & $G$ & Failures & Risk upper \\
\midrule
1 & 22 & 0.08594 & 3.541 & 4.459 & 116/8192 & 1.721\% \\
2 & 29 & 0.11328 & 3.142 & 4.858 & 97/8192  & 1.466\% \\
3 & 18 & 0.07031 & 3.830 & 4.170 & 127/8192 & 1.868\% \\
\bottomrule
\end{tabular}
\end{table}

\paragraph{Results and population transfer.}
All four precision fields vary within each reported family.
For example, seed 2 groups W16/A16/K16/V16 with
W12/A12/K8/V8 while preserving every declared continuation.
Its 29 matching implementations reduce the conditional complexity
upper bound from 8 to 3.142 bits, saving 4.858 bits.

On independent probes, the failure event is disagreement with the
reference by any member of the frozen family at any continuation
position. One-sided Clopper--Pearson bounds use three ledger entries
with $\delta=0.05/3$, giving joint 95\% confidence. The resulting
population-disagreement bounds are 1.47--1.87\%; the predeclared
1\% transfer criterion is not met. Exact equality therefore applies
to the construction domain, while the independent bounds quantify
transfer to the declared finite text population. This experiment
uses a separate confidence budget from the other studies.
\subsection{Complete four-layer decoders and lossless compression}
\label{core:complete}
This experiment certifies every learned parameter of a four-layer causal
decoder. The pre-LayerNorm model has width 16, two attention heads, FFN width
32, and 9,472 parameters. A length-nine input begins with one of eight
selector symbols, followed by eight independent uniform tokens from a
16-symbol alphabet; the label is the selected token. The declared population
has $8\cdot16^8=34{,}359{,}738{,}368$ prefixes. Prefixes are drawn independently
and uniformly with replacement, so Theorem~\ref{thm:behavioral}'s sampling
hypothesis holds with respect to that population.

Linear matrices use W4 symbols $\{-7,\ldots,7\}$ and FP32 matrix scales;
embeddings are FP16 and biases and normalization parameters FP32. Linear
inputs use A4 or A8 dynamic max-absolute quantization. A further quantizer
rounds the token-plus-position embedding sum before the first block, with
no bypass. Its FP32 scale is computed from the complete token--position
alphabet during training and frozen for deployment. This input transition
is part of the trained architecture.

Seeds 191, 193, and 197 are evaluated at both activation precisions.
Each model trains for 16 epochs on 30,000 examples (data seed 7241), using
AdamW with learning rate 0.003, weight decay $10^{-4}$, batch size 512,
gradient clipping at one, and straight-through quantizer gradients.
The certificate sample includes these 30,000 examples and 500,000 fresh
examples (seed 8123); a separate 20,000-example set (seed 1049) measures
held-out performance. Fresh examples do not enter training or selection.

\paragraph{Loss and complete code.}
The deployed distribution is $p_{s,e}=(15/16)p_e+1/256$. Set
$a=\log_2(256/241)$ and $\Delta=8-a$. We certify
\[
g_e^{01}(x,y)=\mathbf1\{\arg\max_k p_{s,e}(k\mid x)\ne y\},\qquad
g_e^{\rm NLL}(x,y)=\frac{-\log_2p_{s,e}(y\mid x)-a}{\Delta}.
\]
Both are in $[0,1]$. The likelihood bound converts back to bits as
$a+\Delta U$ and applies to this smoothed deployment.
A fixed decoder reconstructs the architecture, every learned word, quantizer
scales, activation tag, smoothing constant, and codec. A 32-bit byte-length
prefix followed by an $n$-byte payload has prior mass $2^{-32-8n}$;
summing over all $2^{32}$ possible lengths gives total mass one.
Invalid records decode to a fixed fallback.

The literal codec uses W4 nibbles. The compressed codec stores each matrix
histogram and its exact symbol-order rank, charging
$\lceil\log_2 {N+14\choose14}\rceil+
\lceil\log_2(N!/\prod_jn_j!)\rceil$ bits and its scale. Embeddings remain
literal FP16; remaining values are literal FP32, with byte padding charged.
The literal code costs 58,928 bits and compressed codes 53,560--54,816 bits.
Both codecs reconstruct all parameters exactly and belong to the same prior.

\paragraph{Exact cell and confidence.}
For token embedding $E_{vj}$, position embedding $P_{tj}$, and fixed input
quantizer $Q_{b,s}$, let
\begin{align*}
I_{vj}&=\{h\text{ a finite FP16 word}:Q_{b,s}(h+P_{tj})=
Q_{b,s}(E_{vj}+P_{tj})\ \text{for all }t=0,\ldots,8\},\\
\mathcal C&=\prod_{v=0}^{23}\prod_{j=1}^{16} I_{vj}.
\end{align*}
Position embeddings and downstream fields are fixed. Every member preserves
all input transitions and therefore all logits by induction. This covers
the whole task population, without enumerating its prefixes. Monotone
word-interval searches give $G=\sum_{v,j}\log_2|I_{vj}|$; the compressed
records alone imply $K\le(B_{\rm cmp}-G)\ln2$. There is no transfer term.
Independent enumeration of all finite FP16 words verifies the intervals.

Allocating $0.025$ to each loss event gives joint 95\% coverage across all
six models and all code routes under the common prior, with radius
$(K+\ln40)/530{,}000$. Table~2 reports NLL; \cref{tab:fullmodelerror}
reports error. The BCC NLL bound tightens the compressed bound by
2.13--4.91\% in every run. Complete codecs, checkpoints, per-example losses,
and independent inversions are retained.
\begin{table}[H]
\caption{Complete-model task losses and absolute error certificates.
The 20,000-example held-out set is separate from the certificate sample.
All bounds are simultaneous under the two-loss 95\% ledger and are rounded
upwards; all six declared runs are retained. $B_{\rm cmp}$ is the complete
compressed record length.}
\label{tab:fullmodelerror}
\centering\footnotesize
\setlength{\tabcolsep}{4pt}
\begin{tabular}{@{}llrrrrrr@{}}
\toprule
Format & Seed & $B_{\rm cmp}$ & Test error & Test NLL & $U_{01}^{\rm lit}$ & $U_{01}^{\rm cmp}$ & $U_{01}^{\rm BCC}$\\
\midrule
W4/A4 & 191 & 53,632 & 0.00000 & 0.08855 & 0.07418 & 0.06775 & 0.06383\\
\bccrowsep
W4/A8 & 191 & 54,816 & 0.00040 & 0.09309 & 0.07520 & 0.07020 & 0.06842\\
\bccrowsep
W4/A4 & 193 & 53,560 & 0.00000 & 0.08796 & 0.07418 & 0.06766 & 0.06387\\
\bccrowsep
W4/A8 & 193 & 54,248 & 0.00080 & 0.09678 & 0.07767 & 0.07196 & 0.07006\\
\bccrowsep
W4/A4 & 197 & 53,896 & 0.00000 & 0.08858 & 0.07427 & 0.06816 & 0.06434\\
\bccrowsep
W4/A8 & 197 & 54,280 & 0.00000 & 0.08810 & 0.07418 & 0.06854 & 0.06678\\
\bottomrule
\end{tabular}
\end{table}

\subsubsection{Natural-text feasibility with complete codes}
\label{core:naturalcomplete}
To test whether the same accounting yields useful guarantees on natural
text, we certify six complete four-layer character decoders. The existing
width-16 checkpoints from seeds 1, 2, and 3 of \cref{core:screen} are
each deployed at W4/A4 and W4/A8. The model has two attention heads, FFN
width 32, context length 32, a 65-character vocabulary, and 11,072 learned
parameters. All six choices are fixed before evaluation, with no further
training or calibration search.

\paragraph{Deployment, population, and loss.}
Every linear matrix uses per-row W4 quantization with stored FP32 row
scales; embeddings use FP16 and normalization parameters FP32. Linear
inputs use dynamic per-token A4 or A8 quantization. A stored-scale
quantizer also rounds token-plus-position embeddings before the first
block. Its scale is computed over the complete token--position alphabet
after post-training quantization and then frozen. The complete record
includes the architecture, vocabulary, all learned values, quantizer tags,
scales, smoothing, byte-length prefix, and padding. Raw W4 nibbles cost
89,680 bits; the histogram/order codec costs 87,608--87,680 bits. Both
records reconstruct the same deployed model exactly.

The declared population is uniform over all 1,115,362 length-32
context/next-character pairs in the fixed Shakespeare text. One pair is
one observation. We draw $m=10^6$ population offsets independently with
replacement (seed 3501701), then draw $n=50{,}000$ indices independently
with replacement into that sample (seed 3501702). The deployments are
fixed before this index draw. This is a finite-corpus risk guarantee.
It does not assume independent successive characters or certify an
unseen text distribution.

For $V=65$, the deployed probabilities are
$p_{s,e}=(15/16)p_e+1/(16V)$. Define
$a=-\log_2(15/16+1/(16V))$ and $b=\log_2(16V)$.
The targets are
\[
g_e^{01}(x,y)=\mathbf1\{\arg\max_k p_{s,e}(k\mid x)\ne y\},
\qquad
g_e^{\rm NLL}(x,y)=\frac{-\log_2p_{s,e}(y\mid x)-a}{b-a}.
\]
Both lie in $[0,1]$; the likelihood interval is approximately
$[0.09163,10.02237]$ bits. We compare NLL bounds with the stronger
uniform-predictor baseline $\log_2 65=6.02237$ bits, as well as their
bounded-loss ceiling.

\paragraph{Exact mass and confidence.}
For each token/channel coordinate, form the set of finite FP16 words
preserving its input code at every one of the 32 positions, as in
\cref{core:complete}. Their Cartesian product preserves the entire
input transition for every vocabulary sequence of length at most 32.
All downstream logits, and hence both loss targets, agree. The cell
varies only literal embedding words, so every compressed record has the
same length. Its count gives $G=\sum_{v,j}\log_2|I_{vj}|$ and
$K\le(B_{\rm cmp}-G)\ln2$. Independent exhaustive enumeration of all
63,488 finite FP16 words verifies each interval. No probe-transfer
penalty is needed for this exact cell.

Conditional on the full sample and each fixed
deployment, Hoeffding gives the full empirical normalized risk at most
$q=\min\{1,\widehat R_{\rm sub}+\epsilon\}$, where
$\epsilon=\sqrt{\ln480/(2n)}=0.00785734$, with failure $1/480$.
The twelve model/loss pairs consume $12/480$ confidence; two simultaneous
Occam events, one per loss, each consume $1/80$ and cover every code route
and model. Monotonicity of KL inversion and a union bound therefore give
joint 95\% coverage for
\[
U=\operatorname{kl}^{-1}_+\!\left(q,
\frac{(B-G)\ln2+\ln80}{10^6}\right),
\qquad U_{\rm NLL}=a+(b-a)U.
\]
For raw and compressed controls, set $G=0$ and use their respective
complete record lengths. Raw, compressed, and BCC controls share the
same evaluated losses and confidence allocation.

\begin{table}[!htbp]
\caption{Complete natural-text NLL bounds (bits/character; joint 95\% coverage). One million prefixes are drawn independently with replacement from the declared finite population; 50,000 sampled indices estimate their empirical loss with charged uncertainty. All bounds are below the uniform baseline, $\log_2 65=6.0224$. Empirical denotes the evaluated subsample; values are rounded.}
\label{tab:naturalcomplete}
\centering\footnotesize
\setlength{\tabcolsep}{5pt}
\begin{tabular}{@{}lrrrrrr@{}}
\toprule
Format & Seed & $G$ (bits) & Empirical & Raw & Compressed & BCC\\
\midrule
W4/A4 & 1 & 6840.5 & 3.5146 & 5.3340 & 5.3138 & 5.2446\\
\bccrowsep
W4/A8 & 1 & 2645.6 & 3.3011 & 5.1139 & 5.0936 & 5.0670\\
\bccrowsep
W4/A4 & 2 & 6986.0 & 3.6270 & 5.4485 & 5.4289 & 5.3585\\
\bccrowsep
W4/A8 & 2 & 2589.3 & 3.2965 & 5.1091 & 5.0893 & 5.0633\\
\bccrowsep
W4/A4 & 3 & 6804.5 & 3.4993 & 5.3183 & 5.2978 & 5.2290\\
\bccrowsep
W4/A8 & 3 & 2657.0 & 3.2776 & 5.0894 & 5.0688 & 5.0421\\
\bottomrule
\end{tabular}
\end{table}

Every NLL bound in \cref{tab:naturalcomplete} is below the uniform
baseline; BCC tightens compressed bounds by 0.51--1.30\%. A4 has
larger cells at unchanged stored weight cost, while A8 has lower
empirical NLL and final bounds.

\subsection{Independent transfer: exact, Brier, and clipped-NLL controls}
\label{core:transfer}
One observation is a labeled selector sequence $X=(x,y)$.
The exact target is 0--1 error; the Brier target is
$g_e^{\rm B}(X)=\tfrac12\sum_{k=1}^{16}(p_e(k\mid x)-\mathbf1\{k=y\})^2$.
Both are bounded by one. Checkpoint seeds 89, 97, and 101 and all selection
rules are frozen before the confirming runs. Conditioned on the independent
four-format checkpoint tuple, the prior selects a format uniformly and then
a code uniformly from the center and its one-symbol final-FFN replacements.
A 512-example construction sample selects the minimum-Brier implementation
using fixed tie rules. Exact and approximate transfer use separate untouched
20,000-example probe samples; separate test samples are descriptive.

For 0--1, the posterior conditions the prior on the selected reference's
entire construction-loss vector. For Brier, candidates in the selected
format are sorted by construction mean loss, and the uniform posterior on
the nonempty prefix minimizing the PAC--Bayes core is selected. This need
not be a Brier equality cell: its actual Gibbs loss and independent Brier
discrepancy enter Theorem~\ref{thm:publicquotient}. The selected support is the
whole ternary/A8 format in all three confirming runs. Each loss block uses
failure probability $0.05/9$ per seed/event for the raw comparator,
posterior core, and transfer. \Cref{tab:maincertificates} summarizes the gains;
\cref{tab:selectedquotient} gives every confirming row.
\begin{table}[!t]
\caption{Risk bounds after independent transfer. Decoder ranges span three
seeds; GPT-2 uses the local prior. T denotes ternary and $G$ is the saving
in bits. Protocols and confidence ledgers:
\cref{core:transfer,core:gpt}.}
\label{tab:maincertificates}
\centering\footnotesize
\setlength{\tabcolsep}{4pt}
\begin{tabular}{@{}llrrrr@{}}
\toprule
Audit & Format & Probes & $G$ (bits) & $U_{\rm raw}$ & $U_{\rm BCC}$\\
\midrule
Decoder: exact 0--1 & T/A8 & 20,000 & 12.000 & 0.02868 & 0.02035--0.02060\\
\bccrowsep
Decoder: Brier & T/A8 & 20,000 & 12.000 & 0.02890--0.03022 & 0.02065--0.02215\\
\bccrowsep
GPT-2 FFN & W1/A4 & 1,024 & 13.559 & 0.08062 & 0.07256\\
\bccrowsep
GPT-2 FFN & W1/A8 & 1,024 & 13.930 & 0.07851 & 0.06988\\
\bccrowsep
GPT-2 FFN & T/A4 & 1,024 & 14.820 & 0.08094 & 0.07058\\
\bccrowsep
GPT-2 FFN & T/A8 & 1,024 & 15.116 & 0.07984 & 0.06863\\
\bottomrule
\end{tabular}
\end{table}

\begin{table}[H]
\caption{Fresh data-selected quotient confirmation.  Each block uses the conservative per-event failure probability $0.05/9$ for three seeds and the PAC--Bayes core, public transfer, and raw Occam comparator; this is the predeclared primary allocation for Brier.  $|Q|/U$ is the posterior fraction within the selected format, $\widehat R_S(Q)$ is construction Gibbs loss, $\widehat d_Z$ is public mean loss distortion, $U_{\rm raw}^{\star}$ is the best direct hierarchical certificate, $U_{\rm BCC}$ is the transferred quotient certificate, and Gain is the bound-unit difference $U_{\rm raw}^{\star}-U_{\rm BCC}$.}
\label{tab:selectedquotient}
\centering
\scriptsize
\setlength{\tabcolsep}{3.2pt}
\begin{tabular}{llrrrrrr}
\toprule
Loss & Seed & $|Q|/U$ & $\widehat R_S(Q)$ & $\widehat d_Z$ & $U_{\rm raw}^{\star}$ & $U_{\rm BCC}$ & Gain\\
\midrule
0--1 & 89  & $4{,}097/4{,}097$ & 0 & $9.06\!\times\!10^{-5}$ & 0.02868 & 0.02060 & 0.00808\\
\bccrowsep
0--1 & 97  & $4{,}097/4{,}097$ & 0 & $4.76\!\times\!10^{-7}$ & 0.02868 & 0.02035 & 0.00832\\
\bccrowsep
0--1 & 101 & $4{,}097/4{,}097$ & 0 & 0 & 0.02868 & 0.02035 & 0.00833\\
\midrule
Brier & 89  & $4{,}097/4{,}097$ & $1.71\!\times\!10^{-4}$ & $4.91\!\times\!10^{-5}$ & 0.02958 & 0.02148 & 0.00810\\
\bccrowsep
Brier & 97  & $4{,}097/4{,}097$ & $3.42\!\times\!10^{-4}$ & $1.77\!\times\!10^{-5}$ & 0.03022 & 0.02215 & 0.00807\\
\bccrowsep
Brier & 101 & $4{,}097/4{,}097$ & $3.74\!\times\!10^{-5}$ & $7.11\!\times\!10^{-6}$ & 0.02890 & 0.02065 & 0.00825\\
\bottomrule
\end{tabular}
\end{table}

The bounded-likelihood control uses $g_e=\min\{-\ln p_e(y\mid x),10\}/10$
and transfers a 0--1-selected posterior using this loss's own discrepancy. The
4,000-probe control fails to improve the direct certificate in every format:
for T/A8, 12.000 saved bits are insufficient against a 122.710-bit threshold.
The binary posterior cores already exceed their direct bounds, so additional
probes alone cannot fix them. These outcomes distinguish useful fiber credit
from a guarantee that every empirical cell will pay for its validation.

\subsection{GPT-2 final-FFN local-prior audit}
\label{core:gpt}
The pinned GPT-2-small checkpoint has 124,439,808 parameters. Only block 11's
$768\times3072$ and $3072\times768$ FFN matrices are quantized; other fields
are fixed. W1 and ternary maps use the disclosed per-tensor scales, and the
LayerNorm output and post-GELU input use dynamic per-token A4/A8 rounding.
A fixed FP16 cut stores the final-MLP input for replay. The universe contains
the reference and every one-symbol first-FFN replacement in 32 output
channels chosen by a checkpoint/layer hash before reading text: 24,577
binary or 49,153 ternary implementations, costing 14.585 or 15.585 local bits.

The first 512/1,024/1,024 nonempty train/validation/test entries, independently
tokenized and truncated to 64 tokens, provide construction/probe/test sets.
For target $y$, a fixed public modular map chooses distractor $d(y)$.
The document loss is
$g_e(X)=|T_X|^{-1}\sum_{t\in T_X}\mathbf1\{z_{e,y_t}\le z_{e,d(y_t)}\}$,
with zero for an empty target set. Documents receive equal weight; this is
a pairwise ranking target. A cell preserves every construction-document
loss and the dynamic post-GELU scale. Scale-changing candidate--document
pairs receive discrepancy one on probes, a conservative transfer charge.

All four local-prior bounds tighten after transfer
(\cref{tab:maincertificates}). The savings of 13.559--15.116 bits are
93.0--97.0\% of the local code but only 0.000198--0.000295\% of the full
FFN code. The public fixed reference also admits a tighter point prior.

\scalethreesummarytable

\subsection{Forward-only screening and Hessian controls}
\label{core:screen}
Screening measures whether a layer perturbation retains the full-precision
reference's predictions on fixed unlabeled contexts. The tested decoders
have four pre-normalized attention--MLP blocks, GELU, no dropout, and untied
embeddings. Widths 16 and 32 use two and four heads, FFN widths 32 and 64,
and 11,072 and 38,528 parameters. Only QKV, attention-output, and FFN matrices
are perturbed; other fields and activations stay FP32. Seeds 0--1 are pilots;
seeds 2--7 give twelve confirming references across the two widths.

Tiny Shakespeare supplies a 65-character vocabulary. Contiguous regions
allocate 80\% to training, 5\% to Hessian calibration, 5\% to screening,
and 10\% to audit. Training uses length-32 contexts, batch size 32, 1,500
AdamW steps, learning rate 0.003 decaying to 0.0003, weight decay 0.01, and
unit gradient clipping. Each screen uses 96 natural contexts; unigram and
uniform random contexts are distribution controls. Audits use 2,048
nonoverlapping Shakespeare contexts and, after selection is frozen, 2,048
normalized contexts from \emph{Alice's Adventures in Wonderland}.

For a weight row $w_a$ and candidate $b\in\{2,3,4\}$, set
$s_{a,b}=\|w_a\|_\infty/(2^{b-1}-1)$. The primary screen draws 64 independent
uniform perturbations per block/scale, with coordinate ranges
$[-s_{a,b}/2,s_{a,b}/2]$. If $k_{i,\ell,b}$ draws retain context $i$'s
reference prediction, its score is
\[
\widehat S_{\ell,b}=-\frac1{96}\sum_{i=1}^{96}
\ln\frac{k_{i,\ell,b}+1/2}{65}.
\]
Each perturbed block and its complete downstream network are reevaluated;
full-precision prefixes are cached. The score is a smoothed marginal
retention statistic, rather than a joint fiber-mass estimate.

Each selector minimizes its summed block scores on the same 39 W2/W3/W4
schedules with total widths 10, 12, or 14 bits, corresponding to mean block
precisions 2.5, 3, and 3.5. The complete quantized model is then evaluated;
score additivity is a search heuristic and does not assert layerwise
factorization. HAWQ-style scores use 128 Rademacher Hessian--vector products
on 128 calibration contexts, scoring
$\operatorname{tr}(H_\ell)\|Q_b(W_\ell)-W_\ell\|_F^2/n_\ell$.
Covariance-aware Hessian, directional Taylor, weight-MSE, uniform schedule,
and mean-escape controls are retained. Quantized forward evaluation uses
dequantized floating-point weights and FP32 activations.

Table~\ref{tab:mainscreening} shows comparable observed quality to Hessian selection and better
prediction retention than weight-MSE on both corpora. The mean-escape score
slightly improves on the logarithmic score; the covariance-aware Hessian
has slightly lower Shakespeare NLL increase. Seed-level inference averages
budgets and widths within each confirming seed; 20,000 bootstrap resamples
use those six seeds. The HAWQ sign-test $p$-value is 0.3125 on each corpus,
with no claimed statistical equivalence margin.

Timing uses two CPU threads, identical 96-context sets, and three alternating
repetitions on pilot seed zero. BFMS takes 3.40--6.15 seconds versus
17.58--34.13 for 128 Hessian products; ratios against 32 products are
1.15 and 1.55. Dense width-16 Hessian checks have symmetry residual at most
$1.06\times10^{-15}$ and finite-difference curvature error at most
$4.43\times10^{-7}$. Independent replay checks 1,380 schedule arrays and one
complete candidate per reference. The study supports preprocessing;
random perturbation retention alone does not certify deterministic rounding.
Joint W/A screening is specified as an extension, with its population
certification and large-model cost still to be evaluated.

\begin{table}[H]
\caption{Full-precision screening followed by matched-budget quantization. Each row averages twelve confirming references (six seeds at each width) and three mean block-weight precisions, 2.5, 3, and 3.5 bits. Disagreement is relative to the FP reference; error increase is in percentage points and NLL increase in nats. Alice uses unchanged Shakespeare-selected schedules. Natural-context BFMS with uniform step noise is the fixed extension selector; alternatives are retained controls. The random row is the exact uniform mean over each menu.}
\label{tab:fpoutcomes}
\centering\footnotesize
\begin{tabular}{@{}lrrrrrr@{}}
\toprule
& \multicolumn{3}{c}{Shakespeare} & \multicolumn{3}{c}{Alice}\\\cmidrule(lr){2-4}\cmidrule(lr){5-7}Selector & Disag. (\%) & $\Delta$err. (pp) & $\Delta$NLL & Disag. (\%) & $\Delta$err. (pp) & $\Delta$NLL\\
\midrule
BFMS: natural contexts & 39.61 & 8.27 & 0.3936 & 39.90 & 7.35 & 0.3415\\
\bccrowsep
BFMS: unigram-random & 39.50 & 8.24 & 0.3912 & 39.74 & 7.28 & 0.3392\\
\bccrowsep
BFMS: uniform-random & 40.06 & 8.34 & 0.3901 & 40.40 & 7.37 & 0.3384\\
\bccrowsep
BFMS: Gaussian noise & 39.97 & 8.34 & 0.3906 & 40.38 & 7.31 & 0.3387\\
\bccrowsep
Mean escape: natural contexts & 39.26 & 8.23 & 0.3900 & 39.51 & 7.28 & 0.3393\\
\bccrowsep
Radius-calibrated mass & 39.67 & 8.33 & 0.3887 & 39.99 & 7.39 & 0.3401\\
\bccrowsep
HAWQ-style trace & 40.40 & 8.49 & 0.3974 & 40.38 & 7.55 & 0.3441\\
\bccrowsep
Covariance-aware Hessian & 40.12 & 8.38 & 0.3907 & 40.38 & 7.36 & 0.3402\\
\bccrowsep
Weight-MSE & 45.56 & 11.19 & 0.6340 & 45.72 & 9.68 & 0.5307\\
\bccrowsep
Directional Taylor & 42.38 & 8.81 & 0.4059 & 42.63 & 7.87 & 0.3514\\
\bccrowsep
Uniform allocation mean & 54.24 & 13.83 & 0.7658 & 54.43 & 12.46 & 0.6528\\
\bottomrule
\end{tabular}
\end{table}

\subsection{Allocation, calibration, and independent selection}
\label{core:allocation}
For reference FP32 GPT-2 and document $X$, the certified target is
\[
g_e^{\rm dam}(X)=\frac1{10|T_X|}\sum_{t\in T_X}
[\min\{-\ln p_e(y_t\mid X_{<t}),10\}
-\min\{-\ln p_0(y_t\mid X_{<t}),10\}]_+.
\]
The positive part precedes averaging; empty-target documents have zero loss.
Teacher-JS divides mean tokenwise JS by $\ln2$. Documents have equal weight;
ordinary token-averaged NLL is a separate endpoint. These studies use
candidate-prior certificates, with no additional merging inferred from
similar scores.

Sequential selection reevaluates assembled three-block FFN allocations.
The common original prior is uniform over $2{12\choose3}=440$ family/allocation
choices. Six finalists are frozen before independent draws from the 15,575
nonempty training entries after the first 8,192, truncated to 64 tokens.
Certification draws 4,096 entries and evaluation 1,024 entries independently
with replacement, retaining multiplicities. The original-440 and independent
six-finalist routes have radii 0.00238663 and 0.00133805 and share total failure
0.05. This finite-corpus sampling supports the reported confidence ledger.
\begin{table}[H]
\caption{Fresh certification of six frozen GPT-2 deployments on 4,096 iid draws from the declared finite WikiText population, followed by 1,024 independent evaluation draws. $U_{440}$ uses the full allocation prior; $U_6$ uses the independently frozen finalists. Each prior receives $\delta=0.025$, giving joint 95\% coverage. All bounds target normalized positive clipped-NLL damage; $\Delta\mathrm{NLL}$ is token-averaged evaluation NLL increase. Costs match within each weight family.}
\label{tab:practitionerdeployment}
\label{tab:recentobjectiveablation}
\centering\small
\begin{tabular}{llcrrrr}
\toprule
Weights & Selection rule & Blocks & $U_{440}$ & $U_6$ & $d_{\rm eval}$ & $\Delta\mathrm{NLL}$\\
\midrule
W1 & Positive damage & 3,8,10 & \textbf{0.06479} & \textbf{0.06042} & \textbf{0.04948} & \textbf{0.38234}\\
\bccrowsep
W1 & Teacher-JS & 7,9,10 & 0.06867 & 0.06419 & 0.05288 & 0.39389\\
\bccrowsep
W1 & HAWQ-V2 score & 3,8,9 & 0.06706 & 0.06262 & 0.05151 & 0.40241\\
\bccrowsep
Ternary & Positive damage & 1,6,9 & \textbf{0.04973} & \textbf{0.04586} & \textbf{0.03604} & \textbf{0.25209}\\
\bccrowsep
Ternary & Teacher-JS & 3,6,9 & 0.05041 & 0.04652 & 0.03659 & 0.25623\\
\bccrowsep
Ternary & HAWQ-V2 score & 3,8,9 & 0.05094 & 0.04703 & 0.03742 & 0.25424\\
\bottomrule
\end{tabular}
\end{table}

Positive-damage selection improves the certificate, held-out damage, and
NLL over the tested HAWQ score in both families. The earlier isolated-layer
selection underperformed HAWQ on NLL, motivating complete-model reevaluation.
Teacher-JS improves its fidelity target but has a different ternary NLL
ordering.

For independent scale calibration, each FFN scale becomes
$c\operatorname{mean}|W|$ with
$c\in\{1/2,1/\sqrt2,1,\sqrt2,2\}$; the global multiplier is shared across
selected blocks. Public entries 6,144--6,399 select $c=1/2$ before entries
6,912--7,167 are used for certification/selection. The chosen deployment is
frozen before evaluating entries 7,168--7,679. Fixed-block comparisons use a
two-family prior; sequential comparisons share the 440-choice prior.
Fixed-block certificates improve for both families. Joint W1 selection
reduces held-out damage by 19.2\% and NLL increase by 37.4\% versus
unit-scale/checkpoint-MSE selection. The ternary checkpoint-MSE selector is
better than the calibrated selector on both damage and NLL. The full ledger
and paired intervals appear in \cref{tab:scalecalibrationfull}. A same-sample
scale control pays a five-way selector charge that outweighs its empirical
gain, illustrating the value of independent calibration.

\begin{table}[H]
\caption{Full independent-calibration ledger.  Stage C holds blocks fixed and
stage D reruns three-step interaction-aware selection.  $c$ is the global
scale multiplier, $U_d$ is the simultaneous 95\% positive-damage certificate,
$d_{\rm test}$ is positive damage on the untouched 512-document holdout, and
$\Delta\mathrm{NLL}$ is held-out token-NLL increase.  MSE chooses scale only
from checkpoint weight reconstruction; HAWQ replays its frozen unit-scale
allocation.  Lower is better within a stage and weight family.}
\label{tab:scalecalibrationfull}
\centering
\scriptsize
\setlength{\tabcolsep}{3pt}
\begin{tabular}{cllcrrr}
\toprule
Stage & Weights & Rule & $c$ / blocks & $U_d$ & $d_{\rm test}$ & $\Delta\mathrm{NLL}$\\
\midrule
C & W1 & Public behavioral calibration & $0.5$ / 3,8,10 & \textbf{0.08752} & \textbf{0.04260} & \textbf{0.31213}\\
\bccrowsep
C & W1 & Unit = checkpoint MSE & $1$ / 3,8,10 & 0.09219 & 0.04807 & 0.37546\\
\bccrowsep
C & T & Public behavioral calibration & $0.5$ / 1,6,9 & \textbf{0.07484} & \textbf{0.03365} & \textbf{0.21646}\\
\bccrowsep
C & T & Unit & $1$ / 1,6,9 & 0.07735 & 0.03493 & 0.23384\\
\bccrowsep
C & T & Checkpoint MSE & $\sqrt2$ / 1,6,9 & 0.07852 & 0.03582 & 0.25464\\
\midrule
D & W1 & Calibrated sequential BCC & $0.5$ / 1,7,9 & \textbf{0.11124} & \textbf{0.03696} & \textbf{0.22617}\\
\bccrowsep
D & W1 & Unit = MSE sequential BCC & $1$ / 6,8,10 & 0.12438 & 0.04576 & 0.36152\\
\bccrowsep
D & W1 & HAWQ replay & $1$ / 3,8,9 & 0.12911 & 0.05015 & 0.38679\\
\bccrowsep
D & T & Calibrated sequential BCC & $0.5$ / 1,5,7 & 0.10785 & 0.03532 & 0.23468\\
\bccrowsep
D & T & Unit sequential BCC & $1$ / 3,9,11 & 0.10673 & 0.03435 & 0.26586\\
\bccrowsep
D & T & Checkpoint-MSE sequential BCC & $\sqrt2$ / 7,8,10 & \textbf{0.10064} & \textbf{0.03100} & \textbf{0.17574}\\
\bccrowsep
D & T & HAWQ replay & $1$ / 3,8,9 & 0.10948 & 0.03622 & 0.23731\\
\bottomrule
\end{tabular}
\end{table}

\subsection{Fixed routing and exact full-model mass}
\label{core:routed}
Four top-one routed linear experts have dimensions 2, 3, 4, and 5, with 24
fully enumerated states per expert and routing probabilities $(8,4,2,1)/15$.
An independent integer-weight teacher supplies binary labels. Each of 100
frozen seeds uses 128 construction draws. The target is the unscaled 0--1
loss of the selected expert, so disjoint routing meets
Proposition~\ref{prop:routedfactorization}. The independent expert prior
mixes binary and ternary formats equally and is uniform within each alphabet.
For local behavior $g_u$ its mass is
$p_u(g_u)=M_{u,\mathrm{bin}}/(2\,2^{d_u})+
M_{u,\mathrm{tern}}/(2\,3^{d_u})$; global mass is $\prod_up_u(g_u)$.

The executable code charges a four-bit format header and
$\lceil d_u\log_2q_u\rceil$ bits per expert. Dynamic programming optimizes
complexity at each encoded-bit and empirical-error state, then checks the KL
bound. Eight fixed budget ceilings and thirteen risk tolerances are shared
between direct and quotient coding. The per-seed failure allocation gives
joint 95\% coverage. Table~\ref{tab:mainroutedgrid} reports the entire tolerance grid: full mass
admits at least as many models everywhere, with strict gains at nine
thresholds. At 0.30 it admits 79 versus 75 seeds; seven shared feasible cases
save 2--3 encoded bits. All admitted models satisfy their tolerance on the
exhaustively known population. Rational mass checks and exhaustive replay
on three seeds verify the dynamic program.

\begin{table}[t]
\caption{Exact routed composition: certified feasible seeds out of 100
at every declared risk tolerance. Full mass uses the same candidates and
prior as raw coding; per-seed confidence gives joint 95\% coverage.
\Cref{core:routed} specifies the encoded costs.}
\label{tab:mainroutedgrid}
\centering\footnotesize\setlength{\tabcolsep}{2.5pt}
\begin{tabular}{lrrrrrrrrrrrrr}
\toprule
Tolerance & .100 & .125 & .150 & .175 & .200 & .225 & .250 & .275 & .300 & .325 & .350 & .375 & .400\\
\midrule
Raw & 0 & 0 & 2 & 16 & 26 & 43 & 52 & 64 & 75 & 85 & 92 & 97 & 100\\
\bccrowsep
Full mass & 0 & 0 & 4 & 17 & 30 & 44 & 53 & 65 & 79 & 87 & 94 & 97 & 100\\
\bottomrule
\end{tabular}
\end{table}

\paragraph{Exact composition in routed models.}
Per-layer audits cannot simply be multiplied: editing one layer changes the next layer's inputs, so full-model mass is generally out of reach. Routing is the exception. Four fixed top-one experts factorize full-model mass
(\cref{prop:routedfactorization}). Across 100 seeds, full mass admits
at least as many deployments at all 13 tolerances and more at nine
(\cref{tab:mainroutedgrid}). At 0.30 it admits 79 versus 75 seeds;
seven shared cases use 2--3 fewer bits. The common prior ensures feasibility
inclusion; routing permits expertwise computation
(\cref{eq:routedfeasibility,core:routed}).

\subsection{Head cells, pruning, and sign-bit assurance}
\label{core:head}
GPT-2's original tied head is replaced by a separate quantized head so that
head edits leave embeddings and upstream features fixed. Its
$768\cdot50{,}257=38{,}597{,}376$ symbols are binary or ternary with a fixed
global scale. Final hidden features use A4 or A8 max-absolute rounding.
The complete domain comprises all 23,735 eligible nonempty WikiText-2
training rows with at least two tokens, truncated to 32 tokens. Each row
supplies one final next-token prediction; every row enters certification.

The nominal target is next-token 0--1 error. Edit assurance uses
$g_e^{\rm ret}(X)=\mathbf1\{\hat y_e(X)\ne\hat y_{e_0}(X)\}$ relative to
the fixed quantized head. Both average uniformly over this domain.
Dot products use integer codes and a fixed tie rule; their absolute sums
are bounded by $768\cdot127<2^{24}$. Sensitivity sums round outwards and
independent rational checks verify row budgets. Coordinates are ordered by
increasing maximum activation magnitude, with index tie-breaking; each row
frees the largest prefix satisfying the half-margin condition in
Appendix~\ref{app:coordinateproof}. Zero-sensitivity coordinates remain free under ties.
\begin{table}[H]
\centering\small
\caption{A8 population-certified coordinate cells. Error is exact 0--1 loss on all 23,735 declared rows. $D_*$ counts free head symbols and $G=D_*\log_2q$ is the saving under the full uniform head prior. Changes counts prediction changes after canonicalization; no transfer probes are used.}
\label{tab:newheadcells}
\begin{tabular}{rlrrr}
\toprule
Error & Format & $D_*$ & $G$ (bits) & Changes \\
\midrule
0.82772 & W1/A8 & 50,623 & 50,623.0 & 0 \\
\bccrowsep
0.75770 & T/A8 & 47,676 & 75,564.7 & 0 \\
\bottomrule
\end{tabular}
\end{table}

At A8, the binary bank has 50,623 addresses and the ternary bank contributes
75,564.672 bits of credit. No A8 feature coordinate is identically zero.
Every permitted sign choice preserves predictions; zeroing remains within
the same displacement budgets. Pruning removes 50,623 binary or 45,626
nonzero ternary entries. Ten uniform masks and maximum/RMS-activation masks
match edit count and Frobenius displacement within each alphabet.
The certified masks change no predictions; the RMS masks change 57/91 and
the maximum masks 2,218/72 (Table~\ref{tab:mainedits}). Every random mask changes more than
23,000 predictions. NLL can change and favors RMS pruning in this comparison.

The sign-fault campaign freezes equal-size banks and tests 16,384 patterns,
each flipping 5,063 signs. The certified bank has zero failures, and each
activation-only bank fails every trial. The certificate covers all permitted
subsets, including those untested by injection. Certification cost differs
from computing saliency and is not matched here.

Applying \cref{prop:coordinatecell} yields a certified bank of 50,623 signs
in the GPT-2 W1/A8 output head whose every subset can be flipped
simultaneously without changing any of the 23,735 declared predictions.
All $2^{50{,}623}$ combinations remain within the same prediction-preserving
cell and behavioral fiber, retaining the same 0--1 loss and, with the prior
and certification data fixed, the same behavioral certificate. We evaluate
16,384 patterns, each flipping 5,063 signs; a trial fails if any declared
prediction changes. The certified bank records no failures, while
RMS-activation and maximum-activation banks matched in size, edit count,
and displacement each fail every trial. \Cref{tab:mainedits} also reports
prediction-preserving pruning under the same margin budgets
(\cref{core:head}).

\begin{table}[!t]
\caption{Pruning and fault assurance on 23,735 GPT-2 contexts, with matched
edit counts and displacement. Pruning removes 50,623 W1 or 45,626 ternary
weights; fault trials flip 5,063 signs. Any changed prediction fails
(\cref{core:head,core:head}).}
\label{tab:mainedits}
\centering\footnotesize
\setlength{\tabcolsep}{8pt}
\begin{tabular}{@{}lrrr@{}}
\toprule
& \multicolumn{2}{c}{Pruning: prediction changes} & Sign flips: failed patterns\\
\bccrowsep
Selection & W1/A8 & T/A8 & Out of 16,384 trials\\
\midrule
Certified cell/bank & 0 & 0 & 0\\
\bccrowsep
RMS activation & 57 & 91 & 16,384\\
\bccrowsep
Max activation & 2,218 & 72 & 16,384\\
\bottomrule
\end{tabular}
\end{table}

\subsection{Output-head edit certification on OLMoE}
\label{app:olmoe_head}

\paragraph{Implementation and certified family.}
We evaluate \texttt{allenai/OLMoE-1B-7B-0924} at pinned revision
\texttt{6d84c485}. The model executes its learned top-eight routing
over 64 experts in each of 16 layers. Expert and router weights
remain BF16. A separate output head uses sign-magnitude W8 codes,
fixed dyadic row scales, and H16 feature codes.
Head fitting compares round-to-nearest with GPTQ-style damped
Gram reconstruction using construction data only; GPTQ is selected.
The selected bank contains 27,550 nonzero coordinates and is shared
by BASE, K8/V8, and K8/V4. All $2^{27{,}550}$ sign combinations
are covered by the interval check. Reducing selected magnitudes,
including pruning to zero, remains inside the same score intervals.
The reference for each check is that deployment's own unedited
integer head.

\paragraph{Construction, calibration, and probes.}
The finite population is the 286,892 eligible endpoints in the pinned
WikiText-2 test token stream at dataset revision \texttt{b08601e0}.
Context length is 256, and the continuation horizon is four
teacher-forced and four greedy steps. Independent uniform draws
with replacement supply 512 construction, 1,024 calibration,
and 8,192 probe prefixes, using experiment seed 2026092291.
Head fitting uses 256 construction prefixes; candidate-head selection
uses a separate 128 construction prefixes.
Seven predeclared bank divisors, $1,2,4,8,16,32,64$, are evaluated
on calibration data. The largest eligible bank must have zero
calibration union failures in every deployment and mean positive
clipped-loss envelope at most $0.02$ nats.
Divisor 64 selects the 27,550-coordinate family; the selected family
and matched controls are frozen before probing.

\paragraph{Failure event and confidence ledger.}
A prefix fails if some permitted edit changes some required
teacher-forced or greedy prediction. Tokens within a prefix are
not treated as independent samples.
The fixed ledger contains 168 one-sided confidence entries, including
the cross-deployment union event and the separate win and loss
bounds used in each paired comparison. The experiment-level
failure budget is $1/60$, giving
\[
\delta_{\mathrm{entry}}
=\frac{1}{60\cdot168}
=\frac{1}{10{,}080}.
\]
Thus the ledger has simultaneous confidence at least $98.33\%$.
One-sided Clopper--Pearson upper bounds use the recorded prefix
failure counts: BASE has $5/8192$ failures and upper risk $0.239\%$;
the two primary rows appear in \cref{tab:olmoe_head}.
Across all three deployments, 15 prefixes fail in at least one
setting, giving a simultaneous union-risk upper bound of $0.430\%$.

\paragraph{Quality gates.}
Losses are clipped at 20 nats per token and averaged within each
prefix. Signed head-quality and deployment-quality confidence
bounds use the frozen mixture-betting procedure, with additional
raw empirical mean checks. The head tolerance is $0.1$ nat and
the deployment tolerance is $0.25$ nat. The measured signed head
change is $0.000371$ nats per token, with upper bound $0.026965$.
The other family criteria require at least 1,024 free signs,
prefix failure-risk upper bound at most $1\%$, added clipped-loss
upper bound at most $0.1$ nat, population clipped-loss upper bound
at most $0.95\ln V$, and cache saving at least $40\%$, where
$V=50{,}280$. Both primary settings meet all seven criteria.
BASE is a reference row and has no cache-saving requirement
in the primary verdict.

\begin{table}[t]
\centering\footnotesize
\setlength{\tabcolsep}{3.5pt}
\caption{OLMoE quality diagnostics, in nats per token.
Head upper compares the frozen integer head with the raw head
under BASE. Deployment changes compare each unedited deployment
with BASE using the same frozen head. The measured edit envelope
averages tokenwise worst-edit positive clipped-loss increases;
its population upper bound applies to every head in the family.}
\label{tab:olmoe_quality}
\begin{tabular}{@{}lrrrrr@{}}
\toprule
Setting & \shortstack{Head\\upper} & \shortstack{Deployment\\signed mean}
& \shortstack{Deployment\\signed upper} & \shortstack{Measured\\edit envelope}
& \shortstack{Edit-loss\\upper}\\
\midrule
BASE & 0.026965 & +0.000000 & 0.000000 & 0.0001277 & 0.023285 \\
K8/V8 & 0.026965 & -0.000344 & 0.026288 & 0.0001276 & 0.023284 \\
K8/V4 & 0.026965 & +0.006865 & 0.033662 & 0.0001280 & 0.023287 \\
\bottomrule
\end{tabular}
\end{table}

\paragraph{Matched controls.}
The magnitude and activation-cost controls each contain 27,550
nonzero editable entries. Activation cost ranks entries by
$|c_{ij}|\max_X|z_j(X)|$, computed globally on construction data only.
The paired estimand is control failure probability minus selected
family failure probability. Its lower bound subtracts the
Clopper--Pearson upper bound on control-only successes from the
lower bound on selected-only successes; both events receive their
own entries in the ledger. The lower advantages are positive for
both deterministic controls in all three settings
(\cref{tab:olmoe_controls}). These comparisons concern the event
that some allowed edit changes a prediction. Full-pruning
comparisons and random-control diagnostics are recorded separately.

\begin{table}[t]
\centering\footnotesize
\setlength{\tabcolsep}{4pt}
\caption{OLMoE matched robustness controls. Failure counts are out
of 8,192 prefixes. Lower advantages are in percentage points (pp)
and use the complete simultaneous confidence ledger.}
\label{tab:olmoe_controls}
\begin{tabular}{@{}lrrrrr@{}}
\toprule
& \multicolumn{3}{c}{Failed prefixes}
& \multicolumn{2}{c}{Lower advantage (pp)}\\
\cmidrule(lr){2-4}\cmidrule(l){5-6}
Setting & Selected & Magnitude & Activation cost & Magnitude & Activation cost\\
\midrule
BASE & 5 & 133 & 107 & 0.915 & 0.675 \\
K8/V8 & 4 & 137 & 106 & 0.978 & 0.689 \\
K8/V4 & 6 & 125 & 103 & 0.823 & 0.625 \\
\bottomrule
\end{tabular}
\end{table}

\paragraph{Routing, cache quality, and scope.}
The recorded execution visits all 64 experts in every layer in
both teacher-forced and greedy modes. Teacher-forced prediction
agreement with BASE is $97.7600\%$ for K8/V8 and $94.6350\%$
for K8/V4; selected-expert-set agreement is $89.9611\%$ and
$73.3484\%$, respectively. These are empirical diagnostics.
The edit certificate holds the backbone and router parameters
fixed and covers the declared finite population and continuation
horizon. Expertwise factorization is studied separately in
\cref{core:routed}. The review archive provides tables, frozen
settings, and diagnostics; complete feature arrays, integer heads,
and cell files in the full evidence archive support certificate replay.
\subsection{Output-head edit certification on SmolLM2}
\label{app:smollm2_head}

\paragraph{Implementation and certified family.}
We evaluate \texttt{HuggingFaceTB/SmolLM2-1.7B} at pinned
revision \texttt{effd688a}. Backbone weights and activations
remain BF16. A separate output head uses sign-magnitude W8
codes, fixed dyadic row scales, and H16 feature codes.
The input embeddings remain fixed when the output head is edited.
Construction-only selection chooses the GPTQ-style head over
round-to-nearest rounding.

The selected family contains 21,874 nonzero editable coordinates,
shared by BASE, K8/V8, and K8/V4. Integer interval checks cover
all $2^{21{,}874}$ simultaneous sign patterns without enumerating
them. Pruning and permitted integer magnitude reductions remain
inside the same score intervals. Each deployment is compared
with its own unedited integer-head predictions.

\paragraph{Construction, calibration, and probes.}
The population consists of eligible endpoints in the tokenized
WikiText-103 raw test stream at dataset revision
\texttt{b08601e0}. Context length is 1,024, with four
teacher-forced and four greedy continuation steps.
Independent uniform draws with replacement supply 512
construction, 1,024 calibration, and 8,192 probe prefixes,
using seed 20260923101.
Head fitting uses 256 construction prefixes, and candidate-head
selection uses a separate 128 construction prefixes.

The fixed calibration grid uses divisors
$1,2,4,8,16,32,64$. An eligible candidate must have zero
calibration failures in every deployment and mean positive
clipped-loss envelope at most $0.02$ nats per token.
The smallest candidate has one BASE calibration failure and
zero failures in both primary deployments. The predeclared
smallest-candidate fallback therefore selects divisor 64,
giving 21,874 coordinates. This family and all matched controls
are frozen before probe inference.

\paragraph{Failure event and simultaneous confidence.}
A prefix fails if some permitted edit changes any required
teacher-forced or greedy prediction. Tokens within a prefix
are not counted as independent samples.
The fixed ledger contains 168 one-sided confidence entries,
including the union across deployments and the bounds used
in paired control comparisons. The experiment-level error
budget is $1/120$, giving
$\delta_{\mathrm{entry}}=1/(120\cdot168)=1/20{,}160$.
All ledger bounds therefore hold jointly with confidence
$1-1/120$, approximately $99.17\%$.

\Cref{tab:smollm2_head} reports the selected-family certificate.
Across all three deployments, 11 of 8,192 prefixes fail in
at least one setting, giving a union-risk upper bound of
$0.371\%$. Combining this experiment with the OLMoE experiment
in Appendix~\ref{app:olmoe_head} gives joint confidence at
least $97.5\%$ by a union bound.

\begin{table}[t]
\centering
\footnotesize
\setlength{\tabcolsep}{4pt}
\caption{SmolLM2 output-head edit certification at context
length 1,024. The same 21,874-coordinate family is used in
every row. Risk concerns any permitted edit changing any
required prediction in a fresh prefix. Added-loss bounds
apply to every head in the family and are in nats per token.
BASE is a reference row; K8/V8 and K8/V4 are primary.
Displayed upper bounds are rounded upward.}
\label{tab:smollm2_head}
\begin{tabular}{@{}lrrrr@{}}
\toprule
Cache
& \shortstack{Failed\\prefixes}
& \shortstack{Risk upper\\(\%)}
& \shortstack{Added clipped\\NLL upper}
& \shortstack{Cache saved\\(\%)}\\
\midrule
BASE K16/V16 & 4/8192 & 0.228 & 0.024867 & 0.000 \\
K8/V8        & 0/8192 & 0.121 & 0.024867 & 46.875 \\
K8/V4        & 7/8192 & 0.292 & 0.024873 & 59.375 \\
\bottomrule
\end{tabular}
\end{table}

\paragraph{Quality criteria.}
We use the seven criteria described in
Appendix~\ref{app:olmoe_head}, with vocabulary size
$V=49{,}152$. Losses are clipped at 20 nats per token and
averaged within each prefix. Signed quality bounds use
the fixed mixture-betting procedure, with additional raw
empirical mean checks.

The head's measured signed NLL change is $+0.000455$ nats
per token, with upper bound below $0.028823$.
For K8/V8 and K8/V4, the measured deployment changes are
$-0.000341$ and $+0.005792$ nats per token, with signed
upper bounds below $0.028034$ and $0.034260$, respectively.
The population added clipped-NLL bound is below
$0.02488$ nats per token for every head in the selected
family in either primary deployment.
Both primary deployments meet all seven criteria.

\paragraph{Matched robustness controls.}
Every control contains exactly 21,874 nonzero editable
coordinates. The magnitude control ranks absolute aligned
head coefficients; the activation-cost control ranks
$|c_{ij}|\max_X|z_j(X)|$, using construction features only.
The random control samples uniformly from nonzero coordinates.
All control masks are fixed before probing.

\Cref{tab:smollm2_controls} reports the failure counts.
The simultaneous lower bounds on all-pattern failure-risk
reduction against the activation-cost control are
$0.309$ and $0.318$ percentage points for K8/V8 and K8/V4.
Against the magnitude control, the corresponding lower
bounds exceed $4.320$ and $3.991$ percentage points.
These comparisons support a robustness advantage for
the margin-selected family at matched editable size.

\begin{table}[t]
\centering
\footnotesize
\setlength{\tabcolsep}{4pt}
\caption{Matched SmolLM2 controls: failed prefixes out of
8,192. All-pattern failure means that some permitted
combination changes a required prediction. Full pruning
sets every selected coordinate to zero. The simultaneous
advantage statements in the text concern the all-pattern
event.}
\label{tab:smollm2_controls}
\begin{tabular}{@{}lrrrr@{}}
\toprule
& \multicolumn{2}{c}{All patterns}
& \multicolumn{2}{c}{Full pruning}\\
\cmidrule(lr){2-3}\cmidrule(l){4-5}
Bank & K8/V8 & K8/V4 & K8/V8 & K8/V4\\
\midrule
Margin-selected & 0    & 7    & 0    & 1    \\
Magnitude       & 439  & 424  & 43   & 48   \\
Activation cost & 61   & 69   & 15   & 16   \\
Random          & 4971 & 4990 & 1015 & 1023 \\
\bottomrule
\end{tabular}
\end{table}

\paragraph{Certified freedom and memory.}
The certificate covers simultaneous sign flips, pruning,
and permitted magnitude reductions of the selected
output-head coordinates under the declared population
and continuation horizon. Cache quantization supplies
the measured memory savings, including FP32 scale storage;
the edit certificate bounds the additional prediction
risk from changing the output head.

\subsection{Certified edit capacity on SmolLM2}
\label{app:smollm_capacity}

\paragraph{Question and protocol.}
We test how many output-head coordinates each selection method can
modify jointly while meeting a common 1\% population
continuation-disagreement limit.
The model is SmolLM2-1.7B at revision \texttt{effd688a}.
The backbone retains BF16 weights and activations; the untied
output head uses W8 sign-magnitude codes, dyadic scales, and
H16 integer features.
Construction-only head fitting selects between RTN and a
GPTQ-style candidate using separate fitting and selection subsets;
the GPTQ-style candidate is selected before probing.
The two cache settings are K8/V8 and K8/V4, with packed codes
and FP32 scales.

We draw 512 construction prefixes and 8,192 independent probe
prefixes using separate random streams, uniformly with replacement
from the fixed tokenized WikiText-103 test population
(dataset revision \texttt{b08601e0}).
Each prefix contains 1,024 context tokens.
The primary event is that some permitted head modification changes
at least one token of the setting's own unedited 32-token greedy
continuation.
Exact integer interval tests cover every simultaneous sign assignment
and allowed integer magnitude reduction, including pruning.
Four teacher-forced continuation positions per prefix provide
the loss measurements. At the same model and 1,024-token context as Appendix~B.10, this
separate experiment tests a 32-token greedy continuation, whereas
Appendix~B.10 tests four teacher-forced and four greedy steps;
the experiments use separate confidence ledgers.

\paragraph{Frozen selection and confidence.}
All methods use the same head and construction features.
BCC ranks coordinates using activation costs and prediction-margin
budgets; the controls rank by magnitude or activation cost.
At every tested size, each method selects the same number of
nonzero coordinates.
The complete size grid is
\[
\{512,\ 1{,}024,\ 2{,}048,\ 4{,}096,\ 8{,}192,\
16{,}384,\ 32{,}768,\ 65{,}536,\ 131{,}072\}.
\]
Head fitting, rankings, and the grid are frozen before probing.
A fixed 816-entry ledger allocates error probability $0.05/816$
to each one-sided bound across both planned checkpoints,
all sizes, controls, quality checks, and across-cache events.
This gives simultaneous confidence of at least 95\% and covers
selection of the largest passing grid point.
Prefixes are the independent sampling units.

A size passes when its continuation-disagreement upper bound
is at most 1\%, its worst-edit added clipped-NLL upper bound
is at most 0.1 nat per token, and the head and deployment
quality criteria pass.
The head and deployment tolerances are 0.1 and 0.25 nat,
respectively, applied to signed per-prefix clipped-loss bounds
and the corresponding raw empirical means.
NLL is clipped at 20 nats.
Shared certification additionally bounds the probability of
failure in either cache setting by 1\%.

\paragraph{Certified capacity.}
\Cref{tab:smollm_capacity} reports the largest passing size
on the declared grid.
BCC supports $32\times$ as many coordinates as activation-cost
selection both separately and jointly across cache settings.
For 65,536 BCC coordinates, K8/V8 and K8/V4 have respectively
43 and 45 failed prefixes out of 8,192.
The shared 32,768-coordinate family has 26 prefixes failing
in at least one setting.

\begin{table}[t]
\centering
\footnotesize
\setlength{\tabcolsep}{4pt}
\caption{Largest certified coordinate count on the frozen grid
at a 1\% continuation-disagreement limit.
The shared row bounds failure in either cache setting.
For magnitude selection, $<512$ indicates that even the smallest
tested family failed the certification criteria; smaller sizes
were not evaluated; the grid begins at 512.
Risk bounds use the simultaneous 95\% confidence ledger.}
\label{tab:smollm_capacity}
\begin{tabular}{@{}lrrrr@{}}
\toprule
Requirement
& BCC
& \shortstack{Activation\\cost}
& Magnitude
& \shortstack{BCC risk\\upper bound}\\
\midrule
K8/V8
& \textbf{65,536} & 2,048 & $<512$ & 0.904\%\\
K8/V4
& \textbf{65,536} & 2,048 & $<512$ & 0.935\%\\
Shared across both
& \textbf{32,768} & 1,024 & $<512$ & 0.630\%\\
\bottomrule
\end{tabular}
\end{table}

\paragraph{Comparison at equal family size.}
At 32,768 coordinates, BCC also has lower disagreement than
both matched controls (\cref{tab:smollm_capacity_controls}).
Paired simultaneous lower bounds establish an across-cache
risk reduction of 11.69 percentage points relative to
activation-cost selection and 39.00 points relative to
magnitude selection.

\begin{table}[t]
\centering
\footnotesize
\setlength{\tabcolsep}{4pt}
\caption{Matched-size comparison with 32,768 editable coordinates
per method. Counts are failed prefixes out of 8,192.
The last two columns concern failure in either cache setting.}
\label{tab:smollm_capacity_controls}
\begin{tabular}{@{}lrrrr@{}}
\toprule
Method
& K8/V8
& K8/V4
& \shortstack{Either\\setting}
& \shortstack{Union risk\\upper bound}\\
\midrule
BCC
& \textbf{16} & \textbf{11} & \textbf{26} & \textbf{0.630\%}\\
Activation cost
& 582 & 612 & 1,120 & 15.181\%\\
Magnitude
& 2,258 & 2,356 & 3,412 & 43.758\%\\
\bottomrule
\end{tabular}
\end{table}

\paragraph{Quality and fixed edits.}
All reported head and deployment quality gates pass.
The measured head loss increase is 0.000305 nat per token.
The K8/V8 and K8/V4 deployment increases relative to the
integer-head BF16-cache baseline are 0.000209 and 0.006999 nat.
For the 65,536-coordinate BCC families, the worst-edit added
clipped-NLL upper bounds are 0.028151 and 0.028155 nat per token.
At the shared size of 32,768 coordinates, pruning every selected
weight changes 7 and 5 of 8,192 prefixes under K8/V8 and K8/V4;
flipping every selected sign changes 14 and 11, respectively.

\paragraph{Connection to behavioral complexity.}
Conditional on the fixed backbone, head magnitudes, scales,
cache rules, and arithmetic, a uniform prior over the
$N=98{,}228{,}997$ nonzero head signs assigns a
$D$-coordinate sign family mass $2^{D-N}$.
Relative to one fixed sign assignment, this gives $D$ bits of
conditional complexity credit.
The shared margin-budget family therefore earns 32,768 bits,
or 0.033\% of the $N$-bit sign code, compared with 1,024 bits
for the largest certified activation-cost family on the grid.
Independent probes bound the family's population disagreement;
for prediction-error loss, this controls the transfer discrepancy
in \cref{thm:publicquotient}.
The interval certificate additionally covers pruning and magnitude
reduction, while the stated prior mass counts only the sign assignments.

\subsection{Mass-guided perturbation audits}
\label{core:audit}
Two fixed A8 heads are audited under uniform one-symbol replacement and
uniform zeroing of one nonzero symbol. The targets are prediction disagreement
and positive 0--1 task damage. Their mean averages uniformly over the same
23,735 contexts and under the declared edit law. A population-certified
zero-target subset has known mass $s$; the conditional estimator samples its
complement and multiplies the mean by $1-s$ (\cref{prop:residualaudit}).
For additional error, both estimators additionally condition on nominally
correct contexts and multiply by their known fraction.

A complete integer census fixes ground-truth risks but is hidden from both
estimators. All eight conditions, query budgets 256/1,024/4,096/16,384,
128 initial seeds, and 512 confirmation seeds are frozen. Each audit uses
iid context/edit draws and a one-sided 95\% binomial upper limit weighted by
its stratum mass. Confidence is per frozen audit. At 4,096 queries, MSE is
5.36--28.69 times smaller; the exact variance ratios are 5.66--30.64,
corresponding to 82.3--96.7\% fewer queries at matched expected squared error.
\Cref{tab:residualaudit} retains all conditions and budgets. Construction
and sampler setup costs enter \cref{cor:auditplanning}; elapsed-time savings
are unmeasured. These repeated audits assess sampling efficiency on the
fixed heads and population.

\begin{table}[H]
\centering\small
\setlength{\tabcolsep}{3.3pt}
\caption{Certified mass concentrates a perturbation audit. $s$ is the certified zero-target mass; MSE ratios use 512 fresh audit seeds and 4,096 queries per route. The oracle ratio follows from the exact finite-population census. Final columns show median direct/conditional 95\% upper-bound ratios at every prescribed budget. Larger ratios favor conditioning. Both damage routes exploit known nominal correctness.}
\label{tab:residualaudit}
\begin{tabular}{lllrrrrrrr}
\toprule
Format & Edit law & Target & $s$ & \multicolumn{2}{c}{MSE ratio} & \multicolumn{4}{c}{Upper-bound ratio: queries}\\
\bccrowsep
 & & & & Observed & Oracle & 256 & 1,024 & 4,096 & 16,384\\
\midrule
W1/A8 & Substitute & Prediction & 0.820 & 5.36 & 5.66 & 2.23 & 1.56 & 1.27 & 1.12\\
\bccrowsep
W1/A8 & Substitute & $+$ error & 0.942 & 20.10 & 18.33 & 3.24 & 1.93 & 1.38 & 1.18\\
\bccrowsep
W1/A8 & Zero & Prediction & 0.917 & 12.97 & 12.55 & 3.18 & 1.93 & 1.39 & 1.17\\
\bccrowsep
W1/A8 & Zero & $+$ error & 0.965 & 28.69 & 30.64 & 4.58 & 2.42 & 1.61 & 1.25\\
\bccrowsep
T/A8 & Substitute & Prediction & 0.858 & 6.61 & 7.20 & 2.32 & 1.64 & 1.28 & 1.13\\
\bccrowsep
T/A8 & Substitute & $+$ error & 0.949 & 22.80 & 20.91 & 3.25 & 1.94 & 1.41 & 1.18\\
\bccrowsep
T/A8 & Zero & Prediction & 0.888 & 10.45 & 9.34 & 2.51 & 1.64 & 1.28 & 1.13\\
\bccrowsep
T/A8 & Zero & $+$ error & 0.954 & 25.29 & 23.73 & 3.21 & 1.95 & 1.39 & 1.18\\
\bottomrule
\end{tabular}
\end{table}

\subsection{KV-cache precision and logit-preserving cells}
\label{core:cache}
The GPT-2-small checkpoint and WikiText-2-raw-v1 are pinned to revisions
\texttt{607a30d7} and \texttt{b08601e0}. All twelve blocks retain FP32 learned
weights and ordinary activations. Tokenized validation and test text is
partitioned into disjoint 1,025-token windows; fixed seeds select 32 windows
from each split without replacement. These 64 windows form the complete
declared population. Both cohorts enter cell construction.
Each trajectory has a 128- or 1,008-token prefill and sixteen teacher-forced
cached steps. The bounded loss averages $\min\{\ell_{e,t},16\}/16$ over
those sixteen next-token NLLs. Ordinary NLL and FP16-reference prediction
agreement are reported separately.

\paragraph{Arithmetic and cell construction.}
Formats are FP16/FP16, K8/V8, K8/V4, K4/V8, and K4/V4. For each token/head's
64 channels, low-bit arithmetic uses
$s_b(x)=\max\{\mathrm{fl}_{32}(\|x\|_\infty/(2^{b-1}-1)),\epsilon_{32}\}$,
nearest-even rounding, symmetric clipping, and one FP32 scale. Four-bit
codes are packed two per byte. Attention reads the rounded cache during
prefill and continuation, including the current token. This write-before-read
rule is the audited convention; retained cache and temporary workspace are
accounted separately.

Only the 1,536 first-block K/V projection bias words vary within a cell.
Captured pre-bias projections $u_{j,i}$ are fixed under those edits. For a
low-bit group, freeze every coordinate attaining its nominal maximum $M_j$
on any covered transition. For each remaining bias, find all finite FP32
words in the interval around the nominal value that keep
$|\mathrm{fl}_{32}(u_{j,i}+b_i)|\le M_j$ and preserve its quantized code on
every transition. FP16 intervals preserve each stored half word, including
signed zero. Monotone searches, endpoint checks, and excluded-neighbor
checks establish each interval $I_i$.

Every simultaneous choice from $\prod_iI_i$ retains a fixed maximizer and
bounds all other magnitudes, so dynamic scales and stored codes remain
unchanged. Queries and residual inputs are fixed; first-block attention is
identical. Induction through all later blocks and cache writes gives exactly
the same logits on every declared trajectory. This establishes a population
fiber subset for both clipped NLL and predictions at each cache setting.
Different nominal cache settings can still have different losses.

\paragraph{Code, memory, and controls.}
The common prior charges 32 bits per variable bias and a three-bit format
tag, plus $B_{\rm rest}$ for all fixed checkpoint fields and metadata.
Thus $G_k=\sum_i\log_2|I_i|$ yields
$K(g_{e_k})\le(B_{\rm rest}+49{,}155-G_k)\ln2$. The tag and checkpoint cost
are identical across the five formats. Retained cache values are generated during inference; their memory cost is
reported separately from the learned checkpoint.
Reported credit measures the constructed subsets, whose search rules differ
between FP16 and low-bit formats.

For $T$ retained positions, cache bytes are
$12T\cdot12[64(b_K+b_V)/8+4\mathbf1\{b_K<16\}+4\mathbf1\{b_V<16\}]$.
Table~3 gives the 1,024-position trade-off. At matched 14.625 MiB,
K8/V4 has lower NLL, higher agreement, and 187.927 more credit bits than
K4/V8. FP16 has more credit but
uses more memory; cache selection therefore considers all three criteria.
Torch/NumPy interval checks and simultaneous endpoint/random edits verify
packed writes on the complete domain. Full-model replays verify bitwise
logit equality on the recorded replay windows. The mathematical transition
argument supplies the full-domain guarantee.

\begin{table}[!t]
\caption{Certified coordinate cells. Credits add across K and V, including mixed formats. Variable-word counts are out of 768 per side. At prefixes 128 and 1,008, the constraints cover 9,216 and 65,536 first-block transitions respectively. The low-bit construction freezes all observed maximizer coordinates to preserve dynamic scales; the FP16 construction preserves half words directly.}
\label{tab:kvcachecells}
\centering\footnotesize
\setlength{\tabcolsep}{5pt}
\begin{tabular}{@{}lrrrrr@{}}
\toprule
Side & Bits & $G_{128}$ & $G_{1008}$ & Variable words, 128 & Variable words, 1,008\\
\midrule
K & 16 & 1950.972 & 605.528 & 601 & 251\\
\bccrowsep
K & 8 & 411.894 & 3.700 & 42 & 1\\
\bccrowsep
K & 4 & 619.535 & 10.584 & 42 & 1\\
\bccrowsep
V & 16 & 1602.233 & 474.378 & 588 & 220\\
\bccrowsep
V & 8 & 984.485 & 394.488 & 90 & 46\\
\bccrowsep
V & 4 & 1354.424 & 589.298 & 90 & 46\\
\bottomrule
\end{tabular}
\end{table}

\subsection{Cache precision on Qwen2.5 and SmolLM2}
\label{app:modern_kv}

\paragraph{Matched-memory cache precision on Qwen2.5 and SmolLM2.}
The preference for higher key precision extends to Qwen2.5-1.5B
and SmolLM2-1.7B. At equal packed cache memory, K8/V4 yields
lower mean continuation NLL and higher prediction agreement with
bfloat16 (BF16) than K4/V8 at all three tested context lengths
(\cref{tab:modern_kv}). Exact paired tests support a greater than
one-half probability of a strict per-prefix NLL improvement in
all six comparisons, with joint 95\% confidence.
The quantizer, sampling protocol, and complete comparison statistics
are given below.

\begin{table}[!htbp]
\centering\footnotesize
\setlength{\tabcolsep}{4pt}
\caption{Equal-memory cache allocation on two modern checkpoints.
Each row uses 512 prefixes and 32 teacher-forced continuation tokens.
NLL is in nats per token; agreement is with the BF16-cache reference.
K8/V4 and K4/V8 use identical packed cache bytes within each row.}
\label{tab:modern_kv}
\setlength{\tabcolsep}{3pt}       
\renewcommand{\arraystretch}{1}  
\begin{tabular}{@{}lrrrrr@{}}
\toprule
& & \multicolumn{2}{c}{NLL} & \multicolumn{2}{c}{Agreement (\%)}\\
\cmidrule(lr){3-4}\cmidrule(l){5-6}
Model & Context & K8/V4 & K4/V8 & K8/V4 & K4/V8\\
\midrule
Qwen2.5-1.5B & 1,024 & 2.1533 & 10.6837 & 91.37 & 0.57 \\
Qwen2.5-1.5B & 2,048 & 2.0986 & 10.5060 & 91.56 & 0.49 \\
Qwen2.5-1.5B & 4,096 & 2.0613 & 10.3682 & 92.25 & 0.73 \\
SmolLM2-1.7B & 1,024 & 2.0427 & 2.1955 & 95.41 & 81.61 \\
SmolLM2-1.7B & 2,048 & 1.9885 & 2.1168 & 95.46 & 82.87 \\
SmolLM2-1.7B & 4,096 & 1.9672 & 2.0644 & 95.40 & 84.83 \\
\bottomrule
\end{tabular}
\end{table}

\paragraph{Models, arithmetic, and matched storage.}
We evaluate Qwen2.5-1.5B and SmolLM2-1.7B from
\texttt{Qwen} and \texttt{HuggingFaceTB}, respectively, at pinned
model revisions \texttt{8faed761} and \texttt{effd688a}.
Weights and ordinary activations remain BF16. The cache formats
are K16/V16, K8/V4, and K4/V8. Keys are quantized after rotary
embedding. Keys and values use the same symmetric absmax quantizer,
with one FP32 scale per token and KV head; four-bit codes are packed
two per byte. There is no residual full-precision window.
K8/V4 and K4/V8 occupy exactly the same persistent cache bytes,
including scales, at every tested length. Relative to BF16,
both allocations save $60.94\%$ on Qwen2.5 and $59.38\%$ on
SmolLM2. Storage is measured at $L+31$ cached tokens; transient
dequantization buffers and peak GPU allocation are recorded separately.

\paragraph{Population and evaluation.}
Each checkpoint defines a fixed, tokenized, newline-joined
WikiText-103 test population at dataset revision \texttt{b08601e0}.
Eligible endpoints have 4,096 preceding tokens and 32 following
tokens. We draw 512 endpoints independently and uniformly with
replacement for each checkpoint, using seed 2026092207.
The context lengths $L\in\{1024,2048,4096\}$ share these draws.
Repeated endpoints and overlapping text are retained; independence
concerns the endpoint draws conditional on the finite stream.
Each prefix contributes its mean NLL over 32 teacher-forced tokens.
Prediction agreement compares top-one predictions with the same
checkpoint's BF16-cache reference on those tokens.

\paragraph{Frozen comparisons and confidence.}
The primary event is a strict per-prefix NLL win for K8/V4 over
K4/V8; ties count as non-wins. Six predeclared one-sided
Clopper--Pearson lower bounds share a $0.05$ failure budget,
so each receives $\alpha=0.05/6=1/120$.
For $w$ wins among $n=512$ prefixes, the lower bound is the
$\alpha$ quantile of $\operatorname{Beta}(w,n-w+1)$ when $w>0$,
and zero when $w=0$.
All six lower bounds exceed $1/2$, with simultaneous confidence
at least 95\%. Mean NLL and prediction agreement are empirical
summaries. The archived bootstrap intervals for mean NLL gaps
are pointwise descriptive intervals and do not determine this test.

\begin{table}[t]
\centering\footnotesize
\setlength{\tabcolsep}{3.5pt}
\caption{Additional statistics for the modern cache comparison.
The main NLL and agreement comparison is in \cref{tab:modern_kv}.
$\Delta$NLL is K8/V4 minus BF16, in nats per token.
MiB is the equal cache allocation for K8/V4 and K4/V8.
Wins count strict per-prefix NLL improvements for K8/V4;
the final column is the simultaneous lower win-probability bound.}
\label{tab:modern_kv_details}
\begin{tabular}{@{}lrrrrrr@{}}
\toprule
Model & Context & \shortstack{BF16\\NLL} & $\Delta$NLL & MiB
& Wins & \shortstack{Win lower\\(\%)}\\
\midrule
Qwen2.5-1.5B & 1,024 & 2.1203 & +0.0330 & 11.269 & 512/512 & 99.069 \\
Qwen2.5-1.5B & 2,048 & 2.0614 & +0.0372 & 22.206 & 512/512 & 99.069 \\
Qwen2.5-1.5B & 4,096 & 2.0281 & +0.0332 & 44.081 & 512/512 & 99.069 \\
SmolLM2-1.7B & 1,024 & 2.0366 & +0.0061 & 80.361 & 463/512 & 86.880 \\
SmolLM2-1.7B & 2,048 & 1.9820 & +0.0065 & 158.361 & 450/512 & 84.028 \\
SmolLM2-1.7B & 4,096 & 1.9616 & +0.0056 & 314.361 & 433/512 & 80.369 \\
\bottomrule
\end{tabular}
\end{table}

\paragraph{Magnitude of the cache effect.}
Qwen2.5 exhibits severe degradation under the tested K4/V8
quantizer: mean NLL is $10.37$--$10.68$ nats, compared with
$2.03$--$2.12$ for BF16 and $2.06$--$2.15$ for K8/V4.
SmolLM2 supplies the graded comparison: K8/V4 lowers mean NLL
by $0.097$--$0.153$ nats relative to K4/V8 and retains approximately
95.4\% agreement with BF16. These measurements establish the
allocation preference for the declared cache quantizer; the Qwen
comparison does not isolate the contribution of key outlier channels.

\paragraph{Evidence and scope.}
The archive retains per-token losses and predictions, endpoint draws,
all six paired comparisons, persistent memory counts, timing,
the frozen protocol and confidence ledger, and immutable model
and data hashes. These results measure continuation quality and
prediction agreement. Behavioral cell masses and complexity bounds
were not computed for these two checkpoints.

\subsection{Do weight-edit guarantees survive cache compression?}
\label{core:cacheedits}
Cache compression changes the features and output margins that protect a
weight cell. We recheck the original 50,623-sign binary-head bank, certified
with FP32 upstream features, on the same complete 23,735-row domain at each
cache setting. The model uses the twelve-layer cache arithmetic above, the
separate binary head, and A8 final features. The target is prediction
retention relative to the unedited head at the \emph{same} cache setting.
Cache-only changes relative to FP16 are recorded separately.

A second bank retains 50,594 original addresses using FP16 features alone;
its hash is frozen before compressed-cache evaluation. For integer feature
$z$, nominal row score $s_k=c_k^\top z$, and editable coordinates $J_k$,
the exact sign-cell extrema are
\[
\underline s_k=s_k-\sum_{j\in J_k}(c_{kj}z_j+|z_j|),\qquad
\overline s_k=s_k+\sum_{j\in J_k}(|z_j|-c_{kj}z_j).
\]
For nominal smallest-index winner $w$, the entire cell is safe precisely when
$\underline s_w>\overline s_k$ for $k<w$ and
$\underline s_w\ge\overline s_k$ for $k>w$.
Rows independently attain these extrema, proving necessity and sufficiency.
Pruning and mixed pruning/sign edits stay within the same score intervals.
Every reported failure has an independently replayed integer witness.

The original four-format census gives 0, 0, 14, and 25 unsafe contexts at
FP16/FP16, K8/V8, K8/V4, and K4/V4. An unsafe context admits at least one
allowed prediction-changing edit. The matched-memory extension evaluates
K8/V4 and K4/V8 under one recorded CPU runtime, with four threads and batches
of 128. It gives 18 versus 25 unsafe contexts (\cref{tab:cachecompositionmatched}).
The earlier K8/V4 count of 14 belongs to the earlier runtime. Restricting MKL
to AVX2 reduces first-128-context feature-code differences from 164 to one,
but the full-domain shift is not attributed to a specific arithmetic setting.
The matched ordering is consistent with Table~3; robustness across backends
is untested. The net shift of four is not a statistical uncertainty interval.
\begin{table}[H]
\caption{Matched-memory cache precision and weight-cell safety on all 23,735 contexts.
Both formats use the same recorded CPU runtime with unchanged checkpoint bytes.
K8/V4 gives 18 under the runtime of here;
the matched comparison is 18 versus 25. The full count shift is not yet attributed
to a specific arithmetic setting.
Entries count unsafe contexts for the original 50,623-sign bank, frozen FP16
50,594-sign bank, original 41,725-sign common bank, and extended 38,783-sign bank.
KV memory includes FP32 scales and is normalized to 1,024 retained tokens;
the certification prefixes contain 1--31 tokens. Costs match at every length.}
\label{tab:cachecompositionmatched}
\centering\footnotesize
\setlength{\tabcolsep}{5pt}
\begin{tabular}{@{}lrrrrr@{}}
\toprule
& & \multicolumn{4}{c}{Unsafe contexts} \\
\bccrowsep
Cache & KV MiB & Original & FP16 subset & Old common & Extended common \\
\midrule
K8/V4 & 14.625 & 18 & 18 & 0 & 0 \\
\bccrowsep
K4/V8 & 14.625 & 25 & 25 & 3 & 0 \\
\bottomrule
\end{tabular}
\end{table}

The original four-format intersection contains 41,725 signs (82.42\% of the
original bank). It admits three unsafe contexts at K4/V8. The prospectively
specified extension intersects recertified banks from both new runs and
retains 38,783 signs (76.61\%). This subset is safe across all five tested
formats and both recorded K8/V4 runtimes. For every permitted edited head
$c'$ and every declared context,
$\hat y_{c',h}(X)=\hat y_{c,h}(X)$ at each tested cache/runtime $h$.
Subset inclusion proves inheritance from each constituent bank. Consequently
weight edits add zero prediction damage to each cache's own reference,
including every sign pattern, pruning subset, and their mixtures.

The entire-original-bank flip changes 18 predictions in each matched format;
pruning changes 14 versus 13. These endpoint tests complement the whole-cell
criterion. Reusing the common bank on new domains, quantizers, or backends
requires validation.
Cache compression changes the representations and margins protecting these
edits, so we also test whether the weight-edit guarantee survives changes
to cache precision. The full 50,623-sign bank survives K8/V8; a shared
38,783-sign subset (76.61\%) preserves each setting's unedited predictions
across five tested cache formats and recorded runtimes on the declared
domain (\cref{fig:behavioralfibers}, panel~B). Thus, edits within this shared
subset add no prediction changes to those already caused by cache
compression. The matched 18-versus-25 unsafe-context ordering is consistent
with \cref{tab:mainkvcache},
and prediction preservation permits NLL changes
(\cref{tab:cachecompositionmatched,core:head,core:cacheedits}).

\subsection{Quantization-aware training and approximate behavioral mass}
\label{app:qat_mass}

QAT produces substantially larger
approximate behavioral cells and lower perturbation damage than
 PTQ on the
controlled retrieval task. Across four low-bit formats and three
initialization seeds, mean construction-cell coverage increases
from 3.2\% under PTQ to 86.0\% under QAT. QAT also achieves lower
independently measured damage in every format--seed comparison.

\paragraph{Task, models, and training.}
Each input contains nine tokens: an initial selector identifies
one of the following eight positions, whose token is the prediction
target. Data tokens are sampled uniformly from an alphabet of
size 16. The decoder has two transformer blocks, hidden width 32,
four attention heads, and FFN width 64.

For each of three initialization seeds, we train a full-precision
model and QAT models in the four formats
$\{\mathrm{W1},\mathrm{T}\}\times\{\mathrm{A4},\mathrm{A8}\}$,
where T denotes ternary weights.
Training uses 30,000 examples, eight epochs, AdamW with learning
rate $3\times10^{-3}$ and weight decay $10^{-4}$, batch size 512,
and gradient clipping at one.
QAT uses deterministic quantized forward computations with
straight-through gradients. PTQ applies the corresponding
deterministic quantizers to the seed-matched final full-precision
checkpoint.

Binary weights use $\alpha\,\operatorname{sign}(W)$, where
$\alpha=\operatorname{mean}|W|$.
Ternary weights use
$\alpha\,\operatorname{clip}(\operatorname{round}(W/\alpha),-1,1)$.
Activations use dynamic per-token symmetric quantization with
7 or 127 positive levels for A4 and A8, respectively.

\paragraph{Behavioral loss and edit distribution.}
For a labeled sequence $X=(x,y)$, we use normalized Brier loss
\[
g_e^B(X)
=
\frac12\sum_{k=1}^{16}
\left(p_e(k\mid x)-\mathbf{1}\{k=y\}\right)^2
\in[0,1].
\]
Each QAT or PTQ implementation serves as its own nominal reference
$e_0$. The perturbation distribution $\kappa_{e_0}$ is uniform over
all nontrivial single-symbol substitutions at 64 fixed,
hash-selected weight addresses across the two FFNs.
This gives 64 binary edits or 128 ternary edits.
Each edit changes one stored weight code while holding its scale
and all other codes fixed. The forward computation retains dynamic
activation quantization.

\paragraph{Construction and independent evaluation.}
Construction, validation, and evaluation use independently
generated sets of 256, 1,024, and 2,048 examples, respectively.
For construction set $S$, define
\[
C_{S,\epsilon}(e_0)
=
\left\{
e:
\frac{1}{|S|}
\sum_{X\in S}
\left|g_e^B(X)-g_{e_0}^B(X)\right|
\le \epsilon
\right\},
\qquad
s_\epsilon
=
\kappa_{e_0}\!\left(C_{S,\epsilon}(e_0)\right).
\]
The primary tolerance is $\epsilon=1/1024$.
Coverage is computed over the complete declared single-edit
distribution, excluding the unchanged reference.
It measures preservation of mean construction loss within the
specified tolerance.

Independent validation estimates the conditional absolute
discrepancy within the selected cell.
On the separate evaluation set $T$, single-edit damage is
\[
\widehat d_1
=
\frac{1}{|T|}
\sum_{X\in T}
\mathbb{E}_{e\sim\kappa_{e_0}}
\left[
\left(g_e^B(X)-g_{e_0}^B(X)\right)_+
\right].
\]
Thus coverage and damage are evaluated on separate examples.
The four formats share three initialization seeds, yielding
12 paired format--seed comparisons.

\paragraph{Coverage and perturbation damage.}
Table~\ref{tab:qat_mass_results} reports the per-format results.
QAT achieves higher construction coverage and lower independent
single-edit damage in all 12 paired comparisons.
For W1/A8, mean coverage reaches 96.9\%, compared with 3.6\%
under PTQ. The corresponding ternary coverages are 99.7\%
for A4 and 100.0\% for A8, to the displayed precision.

\begin{table}[t]
\centering
\small
\setlength{\tabcolsep}{4pt}
\caption{
QAT/PTQ perturbation audits, averaged over three initialization
seeds. $R_B$ is independent nominal Brier loss;
$s$ is construction-cell coverage at tolerance $1/1024$;
$\widehat d_1$ is independent mean positive Brier damage
under a single-symbol edit. Each comparison uses the same
quantization format and edit addresses. T denotes ternary weights.
}
\label{tab:qat_mass_results}
\begin{tabular}{@{}lrrrrrr@{}}
\toprule
& \multicolumn{2}{c}{$R_B$}
& \multicolumn{2}{c}{$s$}
& \multicolumn{2}{c}{$\widehat d_1$} \\
\cmidrule(lr){2-3}
\cmidrule(lr){4-5}
\cmidrule(l){6-7}
Format & QAT & PTQ & QAT & PTQ & QAT & PTQ \\
\midrule
W1/A4 & 0.03179 & 0.55200 & 0.474 & 0.000
      & 0.001547 & 0.008371 \\
W1/A8 & 0.00440 & 0.54858 & 0.969 & 0.036
      & 0.000116 & 0.002154 \\
T/A4  & 0.00028 & 0.49127 & 0.997 & 0.010
      & 0.000018 & 0.007001 \\
T/A8  & 0.00198 & 0.48880 & 1.000 & 0.081
      & 0.000024 & 0.001984 \\
\bottomrule
\end{tabular}
\end{table}

Using unrounded per-format means, PTQ single-edit damage is
5.4, 18.6, 385.7, and 83.9 times the corresponding QAT damage
for W1/A4, W1/A8, T/A4, and T/A8, respectively.
These averages include every edit in the declared kernel.
Separate sixteen-edit stress tests, using 24 complete perturbations
per model, also give lower damage under QAT in all 12 comparisons.
These stress measurements assess robustness beyond the single-edit
distribution used to construct the cells.

\paragraph{Dependence on the tolerance.}
Table~\ref{tab:qat_mass_tolerance} evaluates the predeclared
tolerances using the same construction measurements.
The QAT coverage advantage persists at the stricter positive
tolerance $1/4096$, with 11 paired wins and one tie,
and at $1/256$, with 12 wins.
At zero tolerance, both methods have zero nonreference coverage
under the declared edit distribution.

\begin{table}[t]
\centering
\small
\caption{
Construction-cell coverage across predeclared Brier tolerances.
Means are over the 12 format--seed cases.
W/T/L counts QAT coverage wins, ties, and losses against PTQ.
The unchanged implementation is excluded.
}
\label{tab:qat_mass_tolerance}
\begin{tabular}{@{}lrrc@{}}
\toprule
Tolerance $\epsilon$ & QAT coverage & PTQ coverage & W/T/L \\
\midrule
$0$      & 0.000 & 0.000 & 0/12/0 \\
$1/4096$ & 0.693 & 0.000 & 11/1/0 \\
$1/1024$ & 0.860 & 0.032 & 12/0/0 \\
$1/256$  & 0.958 & 0.434 & 12/0/0 \\
\bottomrule
\end{tabular}
\end{table}

\paragraph{Connection to the covered-risk guarantee.}
The measured coverage is the mass term in
\cref{prop:coveredrisk}.
Given simultaneous bounds on nominal risk,
$R(e_0)\le U$, and conditional within-cell discrepancy,
$\mathbb{E}_{e\sim\kappa_{e_0}(\cdot\mid C),X}
|g_e^B(X)-g_{e_0}^B(X)|\le\tau$,
the proposition gives
\[
R_\kappa(e_0)
\le
1-s+s\min\{1,U+\tau\}.
\]
For fixed $U$ and $\tau$ with $U+\tau<1$, greater coverage
strictly lowers this upper bound.
The experiment therefore connects training outcomes to the
behavioral mass entering the perturbation guarantee.
QAT's gains in mass and robustness accompany its improvement
in nominal loss; the comparison establishes their association
under the stated quantization and perturbation rules.

\subsection{Ternary coding followed by behavioral aggregation}
\label{app:ternary_mixture}

Lossless coding can exploit unequal ternary symbol frequencies before
behavioral aggregation is applied. We evaluate these two sources of
saving separately, using a fixed normalized prior and the existing
exact finite-model experiments.

\paragraph{A fixed distribution-sensitive ternary prior.}
BITCOS stores a presence bitmap followed by the signs of nonzero
weights~\citep{georganas2026bitcos}.
For $N$ ternary symbols with $n_0$ zeros, its ideal symbol payload is
$2N-n_0$ bits. Define
\[
\pi_0(w)=\prod_{i=1}^N p_0(w_i),\qquad
p_0(0)=\tfrac12,\quad
p_0(-1)=p_0(+1)=\tfrac14.
\]
Then $-\log_2\pi_0(w)=2N-n_0$ and
$\sum_w\pi_0(w)=1$. To cover both sparse and dense codes, use
\[
\pi_{\mathrm{mix}}(w)
=\tfrac12\pi_0(w)+\tfrac12 3^{-N}.
\]
This mixture selects between two whole-vector source distributions.
Writing $L_0=2N-n_0$ and $L_u=N\log_2 3$, its exact charge is
\begin{align}
L_{\mathrm{mix}}(w)
&=\min\{L_0,L_u\}+1
 -\log_2\!\left(1+2^{-|L_0-L_u|}\right) \notag\\
&\le\min\{L_0,L_u\}+1.
\label{eq:ternary_mixture_charge}
\end{align}
Thus the overhead above the better component is at most one bit
per weight vector. Scales, format choices, and other non-public
deployment information receive their own prior charges.

\paragraph{Separating coding and aggregation.}
Let $\nu$ be the complete-implementation prior and $\mathcal F(g)$
an exact behavioral fiber. Using the main paper's complexity and
gain in bits, define
\begin{align*}
K_{\mathrm{dir}}^{\star}(g)
&=-\log_2\max_{e\in\mathcal F(g)}\nu(e),\\
K_{\mathrm{BCC}}(g)
&=\frac{K_{W,A}(g)}{\ln 2}
=-\log_2\sum_{e\in\mathcal F(g)}\nu(e).
\end{align*}
The direct comparator selects a cheapest individual realization
$e^{\star}\in\arg\max_{e\in\mathcal F(g)}\nu(e)$.
The additional aggregation saving is exactly
\[
\frac{G_{\mathrm{BCC}}(e^{\star})}{\ln 2}
=K_{\mathrm{dir}}^{\star}(g)-K_{\mathrm{BCC}}(g)
=\log_2
\frac{\sum_{e\in\mathcal F(g)}\nu(e)}
     {\max_{e\in\mathcal F(g)}\nu(e)}.
\]
Mixture probabilities are summed over fiber members before taking
logarithms; nonuniform code probabilities are retained individually.

\paragraph{Exact reweighting of the declared finite tasks.}
We reuse the two fixed tasks with seeds 110 and 49, exhaustively
evaluating all 512 binary and 19,683 ternary nine-symbol weight
codes at A4 and A8 on their complete 125-input populations.
The architecture, bias, decoder, and remaining inference rules are
public fixed side information. The four W/A formats receive equal
prior mass, giving a two-bit format tag. Binary weight codes remain
uniform; ternary codes receive the prior above.
Every model and prediction is unchanged by this reweighting.

\begin{table}[htbp]
\centering
\small
\setlength{\tabcolsep}{4pt}
\caption{Coding and behavioral savings on the existing exact tasks.
All charges are in bits, including the W/A tag. Direct columns use
the cheapest realization within the same zero-loss fiber.
The final column is additional aggregation saving under the mixture.}
\label{tab:ternary_mixture_results}
\begin{tabular}{@{}rrrrrr@{}}
\toprule
Task & Members & Uniform direct & Mixture direct
     & Mixture BCC & Extra saving \\
\midrule
110 & 7  & 11.000 & 11.000 & 9.387  & 1.613 \\
49  & 10 & 16.265 & 15.498 & 13.062 & 2.436 \\
\bottomrule
\end{tabular}
\end{table}

These charges differ from the main-text A8 comparison for task 110
($10.000\to9.927$ bits), which fixes A8, uses a one-bit tag over
the binary and ternary formats, and assigns uniform symbol priors.
Here a two-bit tag covers four W/A formats, fibers pool A4 and A8,
and ternary codes use the mixture prior. In this four-format universe,
uniform-prior aggregation costs $9.366$ bits, compared with
$9.387$ bits under the mixture.

For task 49, distribution-sensitive coding saves
$0.767$ bits, and behavioral aggregation contributes a further
$2.436$ bits. Its cheapest ternary member has five zeros among
nine symbols, giving $L_0=13$ bits before the format tag.
Task 110 retains an additional aggregation saving
of $1.613$ bits. Both gains are measured against the best individual
realization under the same mixture prior.
The existing histogram-Huffman comparator also retains additional
aggregation savings: $1.344$ bits in task 110 and $1.954$ bits in
task 49. Its direct/BCC charges are $10.000/8.656$ and
$13.000/11.046$ bits, respectively.

\paragraph{Fixed seed sweep.}
Across the existing ternary task seeds $0$--$49$, jointly considering
A4 and A8, the mixture yields strict additional aggregation savings
in 14 of 50 zero-loss fibers and ties in the remaining 36.
The mean saving over all 50 tasks is $0.450$ bits, with a maximum
of $3.828$ bits. The complete per-task results retain every fiber
member and all direct and aggregated prior masses.

\paragraph{Certificate values.}
For an illustrative zero-loss Occam comparison with $m=512$ IID
draws and per-task/prior failure probability $1/80$, the upper
bound is $U(\mu)=1-(\mu/80)^{1/512}$.
The two tasks and two primary priors share a total failure budget
of $0.05$. Under the mixture, aggregation changes the bounds from
$0.023178$ to $0.021042$ in task 110 and from
$0.029108$ to $0.025901$ in task 49.
The population error of each reported fiber is exactly zero by
exhaustive evaluation; these values illustrate the coding effect
on the certificate.

\paragraph{Compatibility with certified sign freedom.}
Suppose a certified cell permits all $2^D$ combinations of $D$
nonzero sign flips while fixing zero locations, magnitudes, metadata,
and all other codes. Every member has the same probability under
both component priors and their mixture. Let $\nu_{\mathrm{mix}}$
be the resulting complete-implementation prior and write
$L_{\mathrm{mix}}(e_0)=-\log_2\nu_{\mathrm{mix}}(e_0)$,
including common metadata. Then
\[
\nu_{\mathrm{mix}}(C)=2^D\nu_{\mathrm{mix}}(e_0),
\qquad
-\log_2\nu_{\mathrm{mix}}(C)
=L_{\mathrm{mix}}(e_0)-D.
\]
Thus aggregation within a proved sign cell retains exactly $D$
bits of saving after this distribution-sensitive coding step.

\paragraph{When the sparse component helps.}
The component $\pi_0$ charges less than the uniform component exactly
when the zero fraction exceeds $2-\log_2 3\approx41.5\%$.
Deployed ternary LLMs reach up to $51.5\%$
zeros~\citep{georganas2026bitcos}, where $\pi_0$ is preferred.
Our six controlled-decoder checkpoints (seeds 89, 97, and 101 at
A4/A8; $17{,}952$ ternary symbols each) have $28.4$--$29.1\%$
zeros, so the mixture charges the uniform cost plus approximately
one bit per model. All symbols and dequantized weight tensors
round-trip exactly. Scales, unchanged non-symbol parameters, and
byte padding are recorded separately. These counts concern lossless
storage; the savings in Table~\ref{tab:ternary_mixture_results}
concern prior mass shared by equivalent implementations.

\section{Scope, limitations, and reproducibility}
\label{app:limitations}

Future work should evaluate BCC on ultra-large language models and larger
mixture-of-experts architectures across a broader range of weight,
activation, and KV-cache formats. The screening results motivate testing
joint precision allocation, interactions between layers, and end-to-end
search efficiency at larger scales. For MoE models, extending certification
to learned routing and simultaneous changes across experts would broaden
its applicability. Another direction is to construct larger behavioral
families spanning multiple layers and formats, and determine when their
shared prior mass yields useful complete-model certificates and realizable
compression gains. These extensions will also require scalable methods
for cell construction and validation.

The accompanying package provides sources, frozen protocols, scripts,
saved outputs, small-model checkpoints, hashes, and pinned revisions for
public weights and data.

\end{document}